\documentclass[mnsc,nonblindrev]{informs3} 

\OneAndAHalfSpacedXI

\usepackage{natbib}
 \bibpunct[, ]{(}{)}{,}{a}{}{,}%
 \def\bibfont{\small}%

 \usepackage{algorithm}

\newcommand{\R}[0]{\mathbb{R}}
\newcommand{\E}[0]{\mathbb{E}}

\newcommand{\norm}[1]{\left\lVert#1\right\rVert}

\newcommand{\paren}[1]{\left( #1 \right)}
\usepackage{caption}

\usepackage{tikz}

\usepackage{amsfonts}
\usepackage{amsmath}
\usepackage{amssymb}
\usepackage[mathscr]{euscript}
\usepackage{soul}

\usepackage[font=scriptsize]{subcaption}

\usepackage[noend]{algpseudocode}

\usepackage{enumerate}

\let\theoremstyle\relax 
\usepackage{amsthm}
\newtheorem{prop}{Proposition}
\newtheorem{assumption_}{Assumption}

\usepackage{thmtools, thm-restate}

\usepackage{booktabs} 

\definecolor{main}{HTML}{5989cf}    
\definecolor{sub}{HTML}{cde4ff}     
\usepackage[many]{tcolorbox}    	

\newtcolorbox{nicebox}{
    sharpish corners, 
    boxrule = 0pt,
    toprule = 4.5pt, 
    enhanced,
    fuzzy shadow = {0pt}{-2pt}{-0.5pt}{0.5pt}{black!35} 
}

\usepackage{breqn}
\usepackage{enumerate} 

\TheoremsNumberedThrough     
\ECRepeatTheorems

\EquationsNumberedThrough    

\MANUSCRIPTNO{}

\def\phcomments{1}
\newcounter{note}[section]
\renewcommand{\thenote}{\thesection.\arabic{note}}
\newcommand{\note}[2]{\refstepcounter{note}\marginpar{\small\bf \textcolor{red}{#1~\thenote}}$\ll${\sf \textcolor{red}{#1's
 Comment~\thenote:}} {\em \textcolor{red}{#2}}$\gg$}

\ifnum\phcomments=1
\newcommand{\notePH}[1]{\note{PH}{#1}}
\else
\newcommand{\notePH}[1]{}
\fi

\ifnum\phcomments=1
\newcommand{\noteGP}[1]{\note{GP}{#1}}
\else
\newcommand{\noteGP}[1]{}
\fi

\begin{document}


\RUNAUTHOR{Cristian et al.}

\RUNTITLE{Meta-optimization for efficient end-to-end learning}

\TITLE{Efficient End-to-End Learning for Decision-Making: A Meta-Optimization Approach}

\ARTICLEAUTHORS{%
\AUTHOR{Rares Cristian}
\AFF{Massachusetts Institute of Technology, \EMAIL{raresc@mit.edu}} 
\AUTHOR{Pavithra Harsha}
\AFF{IBM Research \EMAIL{pharsha@us.ibm.com}}
\AUTHOR{Georgia Perakis}
\AFF{Massachusetts Institute of Technology \EMAIL{georgiap@mit.edu}}
\AUTHOR{Brian Quanz}
\AFF{IBM Research \EMAIL{blquanz@us.ibm.com}}
} 

\ABSTRACT{%
{\color{black}
End-to-end learning has become a widely applicable and studied problem in training predictive ML models in order to be aware of their impact on downstream decision-making  problems. These end-to-end models often outperform traditional methods that separate the prediction from the optimization steps and only myopically focus on prediction error. However, the computational complexity of end-to-end frameworks poses a significant challenge, particularly for large-scale problems. This is because while training an ML model using gradient descent, each time we need to compute a gradient we must solve an expensive optimization problem. We present a meta-optimization method that learns efficient algorithms to approximate optimization problems, dramatically reducing computational overhead of solving the decision problem in general, an aspect we leverage in the training within the end-to-end framework. Our approach introduces a neural network architecture that near-optimally solves optimization problems while ensuring feasibility constraints through alternate projections. We prove exponential convergence, approximation guarantees, and generalization bounds for our learning method. This method offers superior computational efficiency, producing high-quality approximations faster and scaling better with problem size compared to existing techniques. Our approach applies to a wide range of optimization problems including
deterministic, single-stage as well as two-stage stochastic optimization problems. We illustrate how our proposed method applies to (1) an electricity generation problem using real data from an electricity routing company coordinating the movement of electricity throughout 13 states, (2) a shortest path problem with a computer vision task of predicting edge costs from terrain maps, (3) a two-stage multi-warehouse cross-fulfillment newsvendor problem, as well as a variety of other newsvendor-like problems.}
}

\maketitle

\section{Introduction}


In many operations management problems, several input quantities 
need to be predicted from historical data prior to determining the corresponding best operational decision. Examples include predicting travel times in a vehicle routing problem, the demand distribution in a supply chain inventory optimization problem, among many others. A popular approach is to estimate these quantities using a machine learning model from historical data along with observed contextual features such as seasonal trends, location and price information among others. This machine learning model is used to create a forecast for a new observation which is subsequently used for decision making. In these problems, the quality of the resultant business objectives, such as the overall costs in the supply chain problem, is often the more important performance indicator, compared to the accuracy of the machine learning models, such as the mean square error of the demand estimate. 

For example, consider the classical newsvendor problem in which one needs to forecast future demand. There is a holding cost associated with stocking more than the demand, and a backorder cost associated with each unit stocked below demand. Given perfect information about the distribution, the optimal stocking quantity is known to be a particular quantile dependent on the two costs (see \citealt{arrow1951optimal}). Nevertheless, such closed form solutions are not common in most application settings. For example, once there are multiple products with constraints on the stocking allocation, the optimal solution has no longer a closed form. A predict-then-optimize type of approach may take the following steps. (1) Assume the demand distribution comes from some family of distributions (such as Gaussian) and make a prediction using the mean and variance estimated from the data. Subsequently, one can solve the corresponding stochastic optimization problem. Alternatively (2) one may explicitly learn a distributional forecast (for example, learn the probability of each demand realization) and then subsequently solve the corresponding stochastic  optimization problem to determine the allocation quantity. Clearly such distributional forecasts are necessary since using only a point forecast one cannot a priori know the correct statistic of the demand distribution to target. Indeed, simply making point forecasts by minimizing the mean-squared error between the prediction and the observed uncertainty has the potential to produce poor decisions as discussed for example in \citet{cameron2021perils}.

However, joint prediction and optimization, (also referred to as end-to-end learning), bypasses this issue. The primary goal in this case is to learn a forecast whose corresponding decision has the lowest possible cost (objective function). This goal is exactly incorporated into the training algorithm used to learn the forecasting model.
Consider the newsvendor problem as an example. In this problem one needs to be able to implicitly learn the best forecast to make for  the ``correct quantile'' of the distribution
as that would minimize cost. At the same time, one needs to keep in mind that any model will always make some error. However, not all errors are the same. For example, suppose the backorder cost is twice that of the holding cost. {\color{black} A model only minimizing mean-squared error between its predictions and the target would identify an over- and under-prediction the same. However, a model trained end-to-end would identify the true cost resulting from a prediction, and that over-prediction results in lower cost.}

{\color{black} A significant issue in training end-to-end models comes from the required computational complexity. In particular, the loss function is now the output of a decision-making process (to determine the true cost) which is expensive to evaluate. }
In this paper, we propose a novel approach to end-to-end learning that is more efficient compared to existing methods and achieves similar performance. 
For many classes of these optimization 
problems like linear or quadratic, as well as for stochastic optimization problems like the  resource allocation problems, where there are over and under utilization penalties, we observe that point forecasts are sufficient when using an end-to-end learning approach. Together with the proposed learning method in this paper, we can now efficiently solve computationally intractable stochastic optimization problems 
in an end-to-end sense.








More specifically, this paper presents the following contributions: 

\begin{enumerate}[1)]
\item \textbf{Meta-optimization approach to  End-to-End for Single and Two Stage Problems:} {\color{black} We achieve more efficient end-to-end training by replacing the complex optimization-based loss function with a simpler efficient approximation. We introduce a meta-optimization method to learn better surrogate optimization objectives for which we can compute efficient solutions. }
We refer to this approach as ProjectNet. 
We further incorporate this meta-optimization method into an end-to-end learning framework for solving the single-stage stochastic optimization problem. Furthermore, we present a way to extend our approach to two-stage stochastic optimization problems in section~\ref{section:2-stage}.
\item \textbf{Convergence, Approximation, and Generalization:} {\color{black} We prove bounds on how quickly a sequence generated by the ProjectNet architecture converges as well as bounds on the regret of the solution compared to using the exact original optimization problem. In addition, we also prove generalization bounds on how much data is needed to train the ProjectNet model. }
\item \textbf{Point versus Distributional Forecasts for Stochastic Optimization:} Moreover, we prove that for a large class 
of stochastic optimization problems, if one uses a loss-like objective function, then point forecasts through an end-to-end learning approach produce the same optimal solutions as when making distributional forecasts (the latter is often needed in predict-then-optimize approaches).
\item \textbf{Computational Results on Several Applications:} {\color{black} We present computational results on three problems. (i) An electricity generation and planning problem using real-world data from PJM, an organization coordinating the movement of wholesale electricity around 13 US states. (ii) A shortest path problem requiring sate of the art residual network and convolutional networks to learn from map data. (iii) A two-stage multi-warehouse cross-fulfillment newsvendor problem. We compare our methods along multiple axes. First, performance and accuracy as the amount of training data increases. Second, accuracy as the problem size (number of optimization variables) increases. Third,  running time of our method as problem size  increases. We show that our proposed method produces better (or in the worst case competitive) solutions in terms of average cost, while at the same time the method is computationally faster to train relative to other end-to-end learning methods. That is, the proposed approach consistently scales better in terms of runtime as problem size increases, being 2 to 10 times faster for various problems while retaining the same accuracy. }


We show that our proposed approach is computationally efficient as it allows one to quickly approximate the solution, without explicitly solving the optimization problem itself. When training the predictive forecasting model via gradient descent, existing approaches require solving an optimization problem as a subproblem at each gradient update iteration. In contrast, we only require a simple pass through a neural network.
We illustrate computationally for  a variety of problems that ProjectNet produces better values in terms of the objective function as compared to the projected gradient descent method using the same step size and number of iterations. 
For example, for the matching problem we show computationally that the ProjectNet model produces solutions up to 12.5\% better than gradient descent. 




\end{enumerate}

\subsection{Some Related Literature}

One challenge for a large class of decision making problems (for example, for linear optimization problems) stems from the fact that the gradient of the optimal solution with respect to the predicted quantities, e.g., the cost vector of the decision problem, is zero or undefined. 
Often, one aims to learn the forecasting function by using gradient descent. But if the gradient is zero (as in linear optimization), then end-to-end learning becomes next to impossible to perform. 
The reason why this gradient is zero is because in linear optimization problems, the optimal solution lies in a discrete set of vertices of the feasible region and hence, it is a piece-wise constant function of the cost vector. Nevertheless, linear optimization problems constitute a major class of problems of interest. 
\citet{DOBKIN197996} (among others) have shown that any polynomial-time solvable problem can be formulated as a linear optimization problem. This includes for instance the shortest path problem, maximum matching among many others. Thus, most of the existing literature has focused on the linear optimization case.

\noindent {\bf{End-to-end methods for convex problems.}}
To deal with the lack of differentiability issue for linear optimization problems, \citet{elmachtoub2022smart} construct a convex and differentiable approximation of the objective.  Furthermore, \citet{spotree} extend this framework to training with decision trees, and \citet{liu2022online} extend it to a setting where data and decisions need to be taken online over time.
In the case of unconstrained quadratic objectives, \citet{KAO2009} train a model to directly minimize task loss. 
This work was extended in the OptNet framework of \citet{optnet} to constrained quadratic optimization by analyzing the  optimality (KKT) conditions. In particular, this is performed by calculating the optimal solution (using a traditional solution method) and differentiating through the corresponding optimality (KKT) conditions.
\citet{donti2017task} further applies this methodology to stochastic optimization problems with probabilistic constraints.
Subsequently, \citet{Wilder2019} propose to add a quadratic regularization term to linear optimization problems in order to obtain approximate solutions. 
Furthermore, \citet{agrawal2019differentiable} extend the method of analyzing the optimality (KKT) conditions 
to more general convex optimization problems. For other examples of methods geared towards linear optimization, \citet{Mandi2020} use an interior point method to retrieve the optimal solution and calculate the gradient by differentiating through the log-barrier terms. \citet{Pogancic2020} view the optimization problem as a piece-wise constant function (which happens as we discussed, due to the presence of the zero-gradients) of the cost vector and design a continuous interpolation. \citet{Berthet2020} tackle the problem by adding a stochastic perturbation to the linear objective, producing nonzero gradients of the output. 
Unfortunately, a shortcoming of  these approaches is that they are computationally expensive, as the underlying methods need to solve the nominal decision-making optimization problem (e.g., finding an optimal fulfilment strategy given a demand prediction, or a shortest path problem given a prediction of the  travel times of the edges) at each gradient step.

To remedy this issue, we propose a novel neural network architecture which can learn the optimization problem solution, allowing one to  approximate the solution fast, without explicitly solving the exact optimization problem itself. The main issue though in doing so, arises from ensuring the output of the network is a feasible solution to the optimization problem. To overcome this issue, we use an approximate projection method onto the feasible region after each layer of the neural network. We accomplish this by decomposing the problem into a sequence of projections onto ``simpler'' sets (for example, projecting onto a single constraint).

\noindent {\bf{Learning decisions directly.}}
Other approaches in the literature aim to directly learn a decision function rather a forecast. A significant issue to resolve is ensuring the output satisfies the problem constraints. For example, \citet{Ban2019} consider a newsvendor problem and model the decisions directly as a linear function of the features. In this case, the task of learning can be formulated exactly as a linear optimization problem which can be efficiently solved to optimality.
However, this approach does not allow for more complex decision mappings. Alternatively, \citet{Frerix2019} describe a feasible solution, not by its coordinates, but 
as a convex combination of the vertices and extreme rays describing the feasible region. Unfortunately, the primary downside of this approach lies in the often exponential size of the vertex set. 

Closer to our approach, \citet{Donti2021} transform the output of the learning model into a feasible solution by projecting on the equality constraints, and subsequently performing gradient descent to satisfy the inequality constraints. {\color{black} The method of \cite{qiu2024dual} proposes to ensure feasibility by learn a subset of variables and completing the rest automatically by leveraging fundamental
properties of dual-optimal solutions for their specific use-case of AC power flow problems.}

To ensure feasibility, the method proposed in this paper only performs a sequence of alternating projections onto simpler sets. Additionally, our method trains a surrogate model which explicitly learns and approximates solutions to the optimization problem (this may also be of independent interest). The spirit of the work in \citet{shirobokov2020black} is somewhat similar to ours but within a rather different context. 
That is, \citet{shirobokov2020black} consider simulation problems (a common task within fields in  physics or engineering for instance) rather than optimization problems to solve after the initial forecast. The simulation problems considered in \citet{shirobokov2020black} are often highly expensive to perform, and as a result that paper proposes a surrogate generative network method to approximate the outcome.  There has also been interesting work in creating continuous relaxations of algorithms to make them differentiable. For instance, \citet{petersen2021learning} accomplishes this by introducing continuous relaxations of simple algorithmic concepts such as conditional statements, loops, and indexing, which can be pieced together to describe any algorithm. 

Within the context of inventory optimization problems, the integration of learning and optimization was initially studied for example, in \citet{Ban2019}, \citet{Kallus2020} and \citet{oroojlooyjadid2020applying} for solving the feature-based newsvendor problem. 
For a more complex version of the problem, \citet{qi2020practical} devise an end-to-end method for the multi-period replenishment problem. Unlike the previous  classes of problems we discussed, this is a mixed integer linear optimization problem. \citet{qi2020practical} propose a neural network model which can learn the binary values of which days to make an order for inventory as well as how much to order. The paper does so by pre-calculating the optimal solution (order quantities for each day) to many instances observed in the data. The paper then trains a network to learn the mapping from features to the optimal order quantities.



\noindent {\bf{End-to-end learning for integer problems.}}
Although this is not the focus of our paper, there has also been work in the recent years on end-to-end learning for hard combinatorial problems with integer constraints. In what follows, we only provide as a result a brief discussion on some related literature. \citet{ferber2020mipaal} approach the problem by generating cutting planes, taking the corresponding linear relaxation, and applying the approach of \citet{Wilder2019} to solve the new end-to-end problem. \citet{mandi2020smart} apply the approach of \citet{elmachtoub2022smart} to hard combinatorial problems. \citet{guler2020divide} take a new approach using a divide and conquer algorithm when using a linear model to predict the uncertainty 
This approach is applicable to any optimization problem (with possibly nonlinear constraints) and with uncertainty in a linear objective function. The approach of \citet{paulus2021comboptnet} considers the case of uncertainty in constraints in addition to the objective, but restricts itself to linear constraints.

\noindent {\bf{Learning efficient approximations of hard problems.}}
Another stream of work is ``learning to learn'' methods and meta-learning or meta-optimization (ways to learn algorithms that can solve optimization problems). {\color{black} \cite{li2016learning} learn a policy to solve unconstrained continuous optimization problems. They abstract the notion of an optimization algorithm to be a sequence of updates computed from some function of the objective function, the current location and past locations in the sequence. The paper restricts themselves to function values and gradients. Our approach on the other hand can be viewed as learning to solve a different convex problem instead, one which is faster to solve and provides a solution close to that of the original. See section \ref{sec:approx}.

There has also been focus on learning how to better optimize neural networks.  For instance, learn new methods beyond stochastic gradient descent and its variations for training neural networks. See for example in \citet{sergio2016learning}, \citet{andrychowicz2016learning}, \citet{chen2017learning}, \cite{lv2017learning}, \cite{bello2017neural} and the references therein.
}


In a similar vein, there has been recent work in training neural networks to learn optimal solutions of optimization problems. Much of the work in that area, has focused  on difficult mixed integer linear optimization problems (MILP) since the aim is to provide solutions more efficiently than solving the original MILP. 
For instance, one of the earliest proposals for solving the Travelling Salesman Problem \citep{hopfield1985} was to transform the problem into a labelling problem (which edge should be in the path) and use Lagrange multipliers to penalize the solution's violations of the constraints. However, this method has been shown to be highly unstable and sensitive to initialization \citep{wilson1988}. More effective methods for combinatorial problems on graphs have been developed by training recurrent neural networks using reinforcement learning. Examples include \citet{Vinyals2015} and \citet{bello2016neural}. These works are particularly useful for graph problems in which the RNN decides which next node to visit. However, these approaches are not developed with the end-to-end learning framework in mind. 




\noindent \paragraph{Paper organization:} The rest of the paper is organized as follows. In section~\ref{sec:end-to-end} we formally describe the problem setting and the end-to-end framework. Then in section~\ref{section:projectnet} we describe the ProjectNet architecture and how to apply it to learn forecasts end-to-end. We also investigate several theoretical properties and provide guarantees of our approach.
In section~\ref{section:2-stage} we extend our method to two-stage stochastic optimization problems in which there are first-stage decisions to make after which the uncertainty is realized and a second-stage decision then has to be made. Finally, in section~\ref{section:experiments} we  compare computationally our proposed method relative to other existing approaches. We perform these comparisons for the traditional as well as a two-stage multi-location cross-fulfilment newsvendor problem, a convex cost newsvendor problem and a shortest path problem. 

\section{End-to-end learning framework for single-stage stochastic optimization problems}
\label{sec:end-to-end}

We first formally present the problem class requiring the integration of machine learning  (to forecast uncertain parameters) with optimization problems. 
Consider a convex objective function $g_u(w)$, where $u$ is the uncertain parameter(s) that must be predicted. We assume $g_u$ and its derivative is ``simple'' to evaluate. For instance, for linear objectives $g_u(w) = c^Tw$, the cost vector parameters are $u = c$, while for quadratic objectives, $g_u(w) = q^Tw + w^TQw$, the parameters are $u = (q, Q)$. We define the following optimization problem with multiple types of constraints $h_i(w) \leq 0$, 
for a convex function $h_i$ and linear equality constraints $l_j(w) = 0$:
\begin{equation}
\label{nominal-task}
\begin{array}{lll}
w^*(u) =& \arg \min_{w} & g_u(w) \\ \\ 
&\text{subject to} & h_i(w) \leq 0, \ \quad i = 1, \dots, p_1 \\ \\ 
& & l_j(w) = 0, \quad j = 1, \dots, p_2.
\end{array}
\end{equation}
For ease of notation, we let $\mathcal{P}$ also denote the feasible region of  problem (\ref{nominal-task}). Suppose we are given $N$ data points $(x^1, u^1), \dots, (x^N, u^N)$ with features $x^n \in \R^p$ and realized costs $u^n \in \R^d$ (we discussed examples of what these can represent above). Given a model $f_\theta$ parameterized by some $\theta$, for some out-of-sample data $x$, we make a prediction $f_\theta(x)$ (e.g., through a neural network) and a corresponding decision $w^*(f_\theta(x)) \in \mathcal{P}$. 
Afterwards, for a realized cost vector $u$, we incur overall cost  $g_u(w^*(f_\theta(x)))$. However, $u$ itself comes from an unknown distribution dependent on features $x$. Let $D_x$ be the  distribution of $u$ conditioned on observing features $x$. Then, the optimal decision $w^*(D_x)$ is given by minimizing the expected cost, as given by the following stochastic optimization problem:
\begin{equation}
\label{eq:stochastic}
    w^*(D_x) = \arg \min_{w \in \mathcal{P}} \mathbb{E}_{u \sim D_x} \left[ g_u(w) \right].
\end{equation}

We first present some traditional methods of approaching this problem and later illustrate the differences and the advantages of our end-to-end method.

\paragraph{Traditional approaches: predict-then-optimize.} 

 A traditional method, which we will refer to as a \emph{predict-then-optimize} approach, learns the forecasting function independently of the downstream optimization problem. For example, one may attempt to learn the distribution $D_x$ itself. We briefly discuss two methods to accomplish this. First, we may assume $D_x$ belongs to some class of distributions such as the normal distribution, in which case, one only needs to predict the mean and variance from the data. Using this prediction, we can then solve or approximate the solution to (\ref{eq:stochastic}). A second approach is to assume the uncertainty can only take some finite number of discrete values and predict the probability of each value to occur. Then again, one would solve problem (\ref{eq:stochastic}) using this discrete distribution learned.

On the other hand, the problem simplifies if the objective function is linear as a function of the uncertainty. That is, we can write 
\begin{equation}
    g_u(w) = u^T\Phi(w)
\end{equation}
for some function $\Phi$. The simplest example is when $\Phi(w) = w$ and the objective $g_u(w) = u^Tw$ is fully linear. Another example, is the case where the objective is quadratic in terms of $w$. For example, $g_u(w) = \sum_{i,j} u_{i,j} w_{i} w_{j}$. 
In these cases with the objective function being linear in the uncertainty, we can rewrite the stochastic optimization problem as 
\begin{align}
    w^*(D_x) =&\ \arg \min_{w \in \mathcal{P}} \mathbb{E}_{u \sim D_x} \left[ g_u(w) \right] \nonumber \\ 
    =&\ \arg \min_{w \in \mathcal{P}} \mathbb{E}_{u \sim D_x} \left[  u^T\Phi(w) \right] \nonumber \\ 
    =&\ \arg \min_{w \in \mathcal{P}} \mathbb{E}_{u \sim D_x} \left[  u \right]^T \Phi(w).  
\label{eq:linear}
\end{align}
This implies one only needs to predict the mean of the distribution $D_x$ and this is optimal independent of the optimization problem itself. This may be done simply by minimizing the mean squared error between the forecast and the true in-sample cost. Let the forecast be some $f_\theta(x)$, parameterized by $\theta$ (such as a neural network with weights $\theta$). Then, we aim to minimize
\begin{equation}
\theta_{\text{predict-then-optimize}} = \arg \min_\theta \frac{1}{N} \sum_{n=1}^N \norm{f_\theta(x^n) - u^n}_2^2.    
\end{equation}

However, in all of these cases, making a forecast independent of the optimization problem loses out on important gains. Clearly, making perfect forecasts will result in optimal decisions. But in practice, any forecasting model will incur some error. Within the context of the optimization problem, not all errors result in the same cost. For example, consider the case where the downstream optimization problem is linear. Let $\bar{u}_x$ be the mean of the distribution $D_x$ and consider two possible forecasts $u_1$ and $u_2$. It is often the case that $u_1$ is closer in mean-squared distance to $\bar{u}_x$  but the corresponding decision $w^*(u_2)$ has lower objective function value (that is, $g_{\bar{u}_x}(w^*(u_2)) \leq g_{\bar{u}_x}(w^*(u_1))$). We illustrate this idea in the example of Figure~\ref{fig:pointerror}. In this example we observe that point $w^*(u_2)$ is actually the same as the optimal solution $w^*(\bar{u}_x)$, while $w^*(u_1)$ is not. This is despite the fact that the cost vector $u_2$ has higher distance from $\bar{u}_x$ than $u_1$ does. This suggests that a lower mean-squared error in prediction does not necessarily result in decisions with a lower objective function value.
\begin{figure}[t]
    \centering
    \includegraphics[scale=0.7]{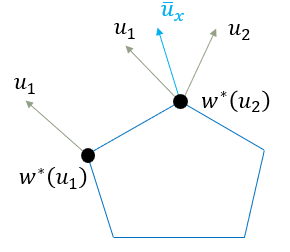}
    \caption{Lower mean-squared error in prediction does not imply lower objective value of the corresponding decisions.}
    \label{fig:pointerror}
\end{figure}
This is due to the fact that a traditional predict-then-optimize method does not make use of this fact when learning its forecasts.

\paragraph{The end-to-end approach.}
As we already discussed, the final goal of an end-to-end approach is to minimize the cost (objective function) of the final decision. We wish to learn a forecasting function $f_\theta(x)$ so that the decisions $w^*(f_\theta(x))$ minimize the expected cost of $g_u(w^*(f_\theta(x))$. This is expressed as follows:
\begin{equation}
\label{eq:end-to-end0}
    \theta_{\text{end-to-end}} = \arg \min_\theta \mathbb{E}_x \mathbb{E}_{u \sim D_x} \left[ g_u(w^*(f_\theta(x)) \right].
\end{equation}
As an approximation, we replace the expectation with samples drawn from the empirical distributions  given by the data. This leads to the following learning task which we call the end-to-end problem we wish to solve.
\begin{equation}
\label{eq:end-to-end}
   \theta_{\text{end-to-end}} = \arg \min_\theta \sum_{n=1}^N g_{u^n}(w^*(f_\theta(x^n)).
\end{equation}
\begin{figure}[b]
    \centering
    \includegraphics[scale=0.5]{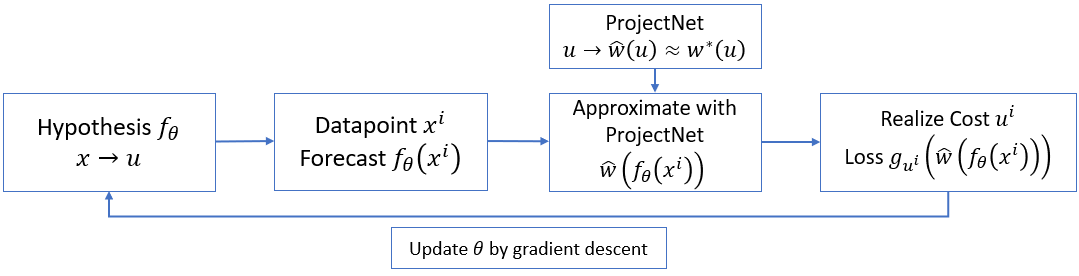}
    \caption{Flow diagram for end-to-end learning using ProjectNet.}
    \label{fig:deterministc_diag}
\end{figure}
See Figure \ref{fig:deterministc_diag}. Given that $f_\theta$ is some neural network, one would generally aim to solve this problem using gradient descent. In order to do so, one must be able to compute the gradient 
\begin{equation}
    \frac{\partial w^*(\nu)}{\partial \nu}
\end{equation}
For each gradient computation one must solve an expensive optimization problem $w^*(f_\theta(x^n))$. Due to the computational time required to solve the optimization problem $w^*(f_\theta(x^n))$ at each gradient step, end-to-end methods easily become intractable as the dimension of the problem increases as this step is significantly more expensive than a simple forward and backward pass normally required to train neural networks.

{\color{black} 
We propose to address this issue by using some approximation $\hat{w}(u)$ which is faster to compute. We show that can safely make this switch. Moreover, we will observe that it suffices for $\hat{w}$ to approximate $w^*$ well only on the distribution of data that we observe. In general, the algorithms used to compute $w^*$ work for any possible input, however this is not necessary in our case. This leads us to learn an approximation $\hat{w}$ that performs well, and potentially more efficiently, on only the distribution of data we observe. 

We first formalize the fact that one can learn using some approximation $\hat{w}$ instead of the exact $w^*$. Let $\hat{f}$ be the model one learns when minimizing the expected cost when using $\hat{w}$ and let $f^*$ be the model when using the true $w^*$. That is,
\begin{align}
    \hat{f} =&\ \arg\min_{f} \mathbb{E}_{x} \mathbb{E}_{u\sim D_x} [g_{u}(\hat{w}(f(x))] \label{eq:opt-fhat} \\
    f^* =&\ \arg\min_{f} \mathbb{E}_{x} \mathbb{E}_{u\sim D_x} [g_{u}({w}^*(f(x))]
\end{align}
Then, having learned $\hat{f}$, we make out-of-samples decisions $w^*(\hat{f}(x))$. The objective value of $w^*(\hat{f}(x))$ compared to the optimal decision $w^*(f^*(x))$ if one had trained using the exact oracle $w^*$ is now bounded only by the approximation error of $\hat{w}$ compared to $w^*$:
\begin{theorem}[Approximation impact on end-to-end] One can decompose the difference in cost of making decisions based off of $\hat{f}$ from the optimal cost from making decisions based off $f^*$ by decomposing the error into two components:
\begin{equation}
    \mathbb{E}_{x} \mathbb{E}_{u\sim D_x} \left [g_{u}(w^*(\hat{f}(x))) - g_{u}(w^*(f^*(x))) \right] \leq \epsilon_1 + \epsilon_2
\end{equation}
where 
\begin{enumerate}
    \item  $\epsilon_1 = \mathbb{E}_x  \mathbb{E}_{u\sim D_x} \left| g_{u}(w^*(f^*(x))) - g_{u}(\hat{w}(f^*(x))) \right|  $. This describes how well $\hat{w}$ approximates $w^*$ on input from $f^*$
    \item $\epsilon_2 = \mathbb{E}_x  \mathbb{E}_{u\sim D_x} \left| g_{u}(w^*(\hat{f}(x))) - g_{u}(\hat{w}(\hat{f}(x))) \right|  $. Similarly, this describes how well $\hat{w}$ approximates $w^*$ on input from $\hat{f}$.
\end{enumerate}
\end{theorem}
Notice that to bound the error terms $\epsilon_1, \epsilon_2$ the approximation $\hat{w}$ only needs to perform well on inputs from $f^*(x)$ and $\hat{f}(x)$. Moreover, while we do not know $f^*$, we do have instances of $f^*(x)$. Specifically, the data $f^*(x_1) = u_1, \dots, f^*(x_N) = u_N$ that we observe.

\begin{proof}
    The proofs follows from a sequence of standard manipulations and triangle inequalities. Consider adding and subtracting the expectation of the term $ g_u(w^*(\hat{f}(x))) $:
    \begin{align}
        &\ \mathbb{E}_{x} \mathbb{E}_{u\sim D_x} \left [g_{u}(w^*(\hat{f}(x))) - g_{u}(w^*(f^*(x))) \right] = \\ &\ \mathbb{E}_{x} \mathbb{E}_{u\sim D_x} \left [g_{u}(w^*(\hat{f}(x))) - \paren{g_u(w^*(\hat{f}(x))) - g_u(w^*(\hat{f}(x)))} - g_{u}(w^*(f^*(x))) \right] 
    \end{align}
    and grouping the first two and second two terms:
    \begin{equation}
        \mathbb{E}_{x} \mathbb{E}_{u\sim D_x} \left [g_{u}(w^*(\hat{f}(x))) - g_{u}(w^*(f^*(x))) \right] \leq \epsilon_2 + \mathbb{E}_{x} \mathbb{E}_{u\sim D_x} \left [ g_u(w^*(\hat{f}(x))) - g_{u}(w^*(f^*(x))) \right].
    \end{equation}
    Now consider adding and subtracting the expectation of the term $g_u(\hat{w}(f^*(x)))$. We find 
    \begin{align}
        &\ \mathbb{E}_{x} \mathbb{E}_{u\sim D_x} \left [g_{u}(w^*(\hat{f}(x))) - g_{u}(w^*(f^*(x))) \right] \leq \\ &\ \epsilon_2 + \mathbb{E}_{x} \mathbb{E}_{u\sim D_x} \left [ g_u(w^*(\hat{f}(x))) - \paren{g_u(\hat{w}(f^*(x))) - g_u(\hat{w}(f^*(x)))} - g_{u}(w^*(f^*(x))) \right].
    \end{align}
    The last two terms in the expectation can be bounded by $\epsilon_1$ so we are now left with
    \begin{equation}\mathbb{E}_{x} \mathbb{E}_{u\sim D_x} \left [g_{u}(w^*(\hat{f}(x))) - g_{u}(w^*(f^*(x))) \right] \leq \epsilon_2 + \epsilon_1 + \mathbb{E}_{x} \mathbb{E}_{u\sim D_x} \left [ g_u(w^*(\hat{f}(x))) - g_u(\hat{w}(f^*(x))) \right].
    \end{equation}
    Finally, note the expectation on the right hand side is always non-positive. Indeed, $\hat{f}$ is the minimizer of $\mathbb{E}_{x} \mathbb{E}_{u\sim D_x} \left [ g_u(w^*(f(x))) \right]$ over $f \in \mathcal{F}$ by definition (see \eqref{eq:opt-fhat}). Therefore, $\mathbb{E}_{x} \mathbb{E}_{u\sim D_x} \left [ g_u(w^*(\hat{f}(x))) \right] \leq  \mathbb{E}_{x} \mathbb{E}_{u\sim D_x} \left[ g_u(\hat{w}(f^*(x))) \right]$.
\end{proof}
}

Moreover, there are additional advantages to using an approximation $\hat{w}$ instead of the exact $w^*$ beyond only speed considerations. Specifically, this approximation aspect is crucial for linear optimization problems --- an incredibly broad class of problems with countless applications. 
As noted earlier, the values of these gradients for commonly occurring linear decision problems are typically zero (as we already discussed above, this is because the optimal solution is always a vertex, and hence a piece-wise constant function, for which the gradient is zero). Moreover, if the gradient is zero this would result in no update of the weights of the neural network, making it impossible to learn. Therefore, using an approximation $\hat{w}$ instead of $w^*$ in solving \eqref{eq:end-to-end} for which the gradient of $\hat{w}$ is non-zero would resolve this issue. 

{See the diagram in Figure \ref{fig:deterministc_diag} for an illustration of our proposed end-to-end method. The exact structure of $\hat{w}$ and how it is learned can be found in section \ref{section:projectnet}. An advantage of our approach is that the computationally expensive optimization problem $w^*(u)$ is never explicitly solved, \textcolor{black}{even during the learning stage to approximate $w^*(u)$}. Rather we replace it with the forward pass of a simple neural network $\hat{w}$ which we will denote as a \emph{ProjectNet} model architecture.} 

Once a forecasting model $\theta_{\text{end-to-end}}$ has been learned, for new out-of-sample features $x$, we make the forecast $u = f_{\theta_{}}(x)$ which then allows us to determine the decision $w^*(u)$ by solving the optimization problem of interest. \textcolor{black}{Depending on whether the forecast is point or distributional in nature, the down stream problem of identifying $w^*(u)$ is a deterministic or a stochastic optimization problem}. We have already shown that making point forecasts is optimal for objective functions that are linear in terms of the uncertainty. Next we show this is also true for nonlinear loss-type objective functions found for example, in inventory and resource management problems. We primarily focus on these classes of problems in the computational section of this paper although the proposed methodology is applicable to convex \textcolor{black}{stochastic} optimization problems in general.

\paragraph{Point Forecasts vs. Distributional Forecasts.} We showed in the case of objectives that are linear in the uncertainty, it is sufficient to predict a point forecast (the mean of $D_x$) instead of the entire distribution $D_x$ itself (see for eq. \eqref{eq:linear}). Using only a predict-then-optimize framework, separating the prediction from the optimization, point forecasts are not sufficient for other more complex objectives. As an example, consider the newsvendor problem. In this case, there is unknown demand $u$ and a decision $w$ to be made on the stock to allocate. There is a holding cost $h$, for every unit of stock unsold and a backorder cost $b$, for every unit of demand that is unmet. The objective function can be represented as 
\begin{equation}
    g_u(w) = \max (h (w - u), b (u - w)),
\end{equation}
which is a nonlinear objective. In this problem, it is known that the optimal stocking decision is the $(b / (h + b))^{th}$ quantile of the demand distribution $D_x$ (see \citealt{arrow1951optimal}). However, a predict-then-optimize point forecast would incorrectly target the mean.

We show that for end-to-end learning problems, point forecasts are sufficient for a large class of convex stochastic problems.
\begin{prop}
\label{prop:point vs. distribution}
Consider a single-stage stochastic optimization problem with a loss-type objective function $g_u(w)$. That is, objectives for which $u$ and $w$ belong to the same feasible space and 
\begin{equation}
    u = \arg\min_{w \in \mathcal{P}} g_u(w).
\end{equation}
Then, for any distribution $D$ of the uncertainty $u$, there exists a single point forecast $d$ which produces the same solution as solving the original stochastic problem:
\begin{equation}
    \arg\min_{w \in \mathcal{P}} \mathbb{E}_{u \sim D} \left[ g_u(w) \right] = w^*(d).
\end{equation}
\end{prop}
An end-to-end method with objective as in \eqref{eq:end-to-end0} would find exactly this forecast $d$ at optimality. 
\begin{proof}
Let $d$ be the solution to the problem using distributional forecast $D$:
\begin{equation}
    d = \arg \min_{w} \E_{u \sim D} [ g_u(w) ].
\end{equation}
Now consider making a forecast of exactly $d$. Then, $w^* = \arg \min_{w} g_d(w) = d$, since $g_d$ is a loss function. 
\end{proof}

It is a major advantage to be able to restrict ourselves only to point forecasts without losing any decision-making ability compared to when making distributional forecasts. This is because point forecasts require far fewer parameters to predict. Moreover, there may be an exponentially large number of possible scenarios, making the nominal optimization problem intractable with distributional forecasts. In contrast, with the right forecast via end-to-end learning, the nominal problem is an easy to solve deterministic model with the point forecast.

A large class of resource management problems where over-utilization and under-utilization are both penalized fall in this category. Note that specifically for the unconstrained newsvendor problem, the work of ~\citet{Ban2019} shows that quantile regression based point forecasts are sufficient. As additional examples, the capacitated newsvendor problem and the multi-location cross-fulfilment newsvendor problem both fall under this category of loss-type functions. In this paper, we present computational experiments on both and more in section \ref{section:experiments} and appendix \ref{app:exp}.

\section{Meta-Optimization Framework}
\label{sec:approx}

{\color{black}
We find that we have two competing objectives when choosing some approximation $\hat{w}$: (1) the objective that $\hat{w}(u)$ is close to $w^*(u)$, and (2) that $\hat{w}(u)$ can be computed significantly more quickly than $w^*(u)$.  At a high level, our meta-optimization method consists  of learning to solve a different version of the original optimization problem, one which is faster to solve and provides a solution close to that of the original. This is similar in spirit to meta-learning methods, however existing work has largely focus on learning better algorithms to provide faster approximations instead of learning better objectives (for e.g. learning accelerations to gradient descent).

First, we formalize what we mean by being able to compute $\hat{w}$ ``faster.'' This is highly dependent on the method used to solve a given optimization problem. Here we will consider iterative interior point methods. And an approximation $\hat{w}$ is computed ``faster'' depending on the number of iterations needed. We consider the following sets of algorithms to produce approximations $\hat{w}$:
\begin{align}
    \hat{w}_{r,T}(u) =&\ w_T, \\ 
    w_{t+1} =&\ \pi_{\mathcal{P}} \paren{ w_t - \eta \cdot R(u, w_t) }, \quad t = 1, \dots, T-1 \\ 
    & \text{initialized } w_0 \in \mathcal{P}
\end{align}
where $\pi_{\mathcal{P}}$ is the projection operator onto the convex feasible region $\mathcal{P}$, and $R(u,w)$ is any update rule. For example, if $R(u,w_t) = \nabla g_u(w_t)$, the above algorithm reduces to projected gradient descent. Finally, let $\hat{w}_R(u)$ denote the limit of this sequence $\hat{w}_{R,T}(u)$ as $T \to \infty$.
Now, we can restate the two objectives when choosing the approximation function $\hat{w}$:
\begin{enumerate}
\item Minimize the regret of $\hat{w}_{R}(u)$ compared to $w^*(u)$. Specifically, minimize $\mathbb{E}[g_u(\hat{w}_{R}(u)) - g_u(w^*(u))]$.
\item Maximize the convergence rate of $\hat{w}_{R,T}(u)$ to $\hat{w}_{R}(u)$.
\end{enumerate}
}

{\color{black} 
Now we gain some more intuition on the effect the update rule $R$ has on the convergence of $\hat{w}_{R}(u)$. Suppose there exists some function $r_u(w)$ for which $\nabla r_u(w) = R(u,w_t)$. Then, $\hat{w}_{R,T}(w)$ essentially performs gradient descent to minimize the function $r_u(w)$ over $w \in \mathcal{P}$. From an analysis point of view, we can instead consider the problem 
\begin{equation}
   \arg\min_{w \in \mathcal{P}} r_u(w)
\end{equation}
and now the point of convergence $\hat{w}_R(u)$ is the solution to the above problem.

In this paper we will make the choice that the update rule is simply the gradient of $g_u$ plus some additional linear term $L(u)$. Therefore, we will choose 
\begin{equation}
    R(u,w) = \nabla g_u(w) + \frac{\gamma}{\eta} \cdot L(u)w
\end{equation}
where $L(u)$ is a matrix, possibly dependent on $u$. The update step then takes the form $\hat{w}_{t+1} = \pi_{\mathcal{P}} \paren{\hat{w}_{t} - \eta \nabla g_u(w) - \gamma \cdot L(u)w}$. For ease of notation, we let $\hat{w}_{L,T}(u)$ denote the sequence of points using this update rule and $\hat{w}_{L}(u)$ be the point of convergence (assuming it exists, more on this in the next section). From this, we see that $\hat{w}_{L}$ is equal to
\begin{equation}
\label{eq:wl}
    \hat{w}_{L}(u) = \arg \min_{w \in \mathcal{P}} g_u(w) + \frac{\gamma}{2\eta} w^T L(u) w.
\end{equation}


\subsection{Analytical properties}
\label{sec:analytical}

First, note that the matrix $L(u)$ needs to be positive semidefinite in order for the  problem in \eqref{eq:wl} to be convex as well as for the sequence to converge properly. In terms of analysis, we will assume that $L(u)$ is positive semidefinite. We will see algorithmically how we might be able to ensure this in section \ref{sec:control}. 

This choice of update rule has some nice properties we can make use of to directly answer objectives (1) and (2) mentioned above. In terms of approximation error, we find the following.
\begin{prop}
\label{prop:bound1}
The decision $\hat{w}_L(u)$ when evaluated against the true solution $w^*(u)$ has regret 
\begin{align}
    g_u(\hat{w}_L(u)) - g_u(w^*(u)) \leq &\ \frac{\gamma}{2\eta} (w^*(u))^T L w^*(u) \\ 
    \leq &\ \frac{\gamma}{2\eta} \sigma_{\max}(L) \cdot D^2
\end{align}
where $\sigma_{\max}(L)$ is the largest eigenvalue of $L$ and $D$ is the diameter of the feasible region $\mathcal{P}$.
\end{prop}
\begin{proof}
    This follows directly from the optimality of $\hat{w}_L(u)$ with respect to the objective $g_u(w) + \gamma/2\eta \ w^TLw$. So, its objective value is lower than the objective value of $w^*(u)$:
    \begin{align}
        g_u(\hat{w}_L(u)) + \frac{\gamma}{2\eta}\hat{w}_L(u)^TL\hat{w}_L(u) \leq&\ g_u(w^*(u)) + \frac{\gamma}{\eta} w^*(u)Lw^*(u) \\ 
       \implies \  g_u(\hat{w}_L(u)) - g_u(w^*(u)) \leq&
        \frac{\gamma}{2\eta} w^*(u)Lw^*(u)  - \frac{\gamma}{2\eta}\hat{w}_L(u)^TL\hat{w}_L(u) \\ 
      \implies \   g_u(\hat{w}_L(u)) - g_u(w^*(u)) \leq& 
        \frac{\gamma}{2\eta} w^*(u)Lw^*(u) 
    \end{align}
Finally, a standard property of eigenvalues states that $\max_{\norm{w}=1} w^TLw = \sigma_{\max}(L)$. Therefore, since $w \in \mathcal{P}$ and hence $\norm{w} \leq D$, it follows that $\max_{w \in \mathcal{P}} w^TLw \leq \sigma_{\max}(L) \cdot D^2$
\end{proof}

In terms of convergence rate, it is known that projected gradient descent has an exponential rate for strongly convex and smooth objective functions. We will make the single assumption that $g_u(w)$ is $\alpha$-smooth, but not necessarily strongly convex. This condition is relatively mild, bounding the gradient of $g$ from above. Most traditional objective functions satisfy this property such as linear, piece-wise linear, quadratic, etc.

\begin{assumption}[smoothness]
\label{ass:smooth}
    We assume the objective function $g_u(w)$ is $\alpha$-smooth and $\beta$-Lipschitz. That is, for any $u$ and any $w^1,w^2 \in \mathcal{P}$ we have
\begin{align}
    \norm{\nabla g_u(w^1) - \nabla g_u(w^2)} \leq \alpha \norm{w^1 - w^2}, \\ 
    \norm{g_u(w^1) - g_u(w^2)} \leq \beta \norm{w^1 - w^2}.
\end{align}
\end{assumption}
\noindent These are often reasonable assumptions. For example, any continuous convex function is Lipschitz over a closed feasible region. 


\begin{prop}
\label{prop:eigen}
    The objective values of the sequence $w_1, \dots, w_t, \dots$ converge exponentially as
    \begin{equation}
        r_u(w_t) - r_u(w^*(u)) \leq \paren{1 - \frac{\gamma/\eta \cdot \sigma_\min(L)}{\gamma/\eta \cdot \sigma_\max(L) + \alpha}}^t \paren{r_u(w_0) - r_u(w^*(u))}
    \end{equation}
    where $\sigma_{\min}(L), \sigma_{\max}(L)$ are the smallest and largest eigenvalues of $L$, and $g_u$ is $\alpha$-smooth.
\end{prop}
\begin{proof}
    We will rely on traditional results in convex optimization for convergence rates for strongly convex and smooth convex functions. Specifically, for $\mu$-strongly convex and $\beta$-smooth convex functions $r_u(w)$, a sequence $w_{t+1} = w_{t} - (1/\beta )\nabla r_u(w_t)$ will converge as follows (see for example, \cite{boyd2004convex}):
    \begin{equation}
        r_u(w_t) - r_u(w^*(u)) \leq \paren{1 - \frac{\mu}{\beta}}^t \paren{r_u(w_0) - r_u(w^*(u))}.
    \end{equation}
    Therefore, it suffices to show that $\mu \geq \gamma/\eta \cdot \sigma_{\min}(L)$ and $\beta \leq \gamma/\eta \cdot \sigma_{\max}(L) + \alpha$ for $r_u(w) = g_u(w) + \gamma/2\eta \cdot w^T L w$. 
    
    We prove strong convexity first. $r_u$ is $\mu$-strongly convex if and only if the smallest eigenvalue of $ \nabla^2 r_u(w) $ is at least $\mu$. We can rewrite this as $\sigma_{\min}(\nabla^2 r_u(w)) = \sigma_{\min}(\nabla^2 g_u(w) + \gamma/2\eta \cdot \nabla^2 w^TLw) \geq \sigma_\min(\nabla^2 g_u(w)) + \gamma/\eta \cdot \sigma_{\min}(L) $. Since $g_u(w)$ is convex, the hessian is positive semidefinite. Therefore, $\sigma_\min(\nabla^2 g_u(w)) \geq 0$. Finally, $\mu \geq \gamma/\eta\cdot \sigma_\min(L)$.

    Next, we prove the smoothness condition. This requires bounding the maximum eigenvalue of $\nabla^2 r$. We have $\sigma_\max(r_u(w)) = \sigma_\max(\nabla^2 g_u(w) + \eta/2\gamma w^TLw) \leq \sigma_\max(\nabla^2 g_u(w)) + \gamma/\eta\cdot  \sigma_\max(L)$. The first term can further be bounded by the smoothness condition. So, $\sigma_\max(r_u(w)) \leq \alpha + \gamma/\eta \cdot \sigma_\max(L)$.
\end{proof}

Putting these two propositions together suggests learning the function $L(u)$ which minimizes the upper bound on regret in proposition \ref{prop:bound1} on the data available while keeping the eigenvalues of $L(u)$ bounded as in proposition \ref{prop:eigen}. Consider solving
\begin{align}
    \min_{L} &\ \sum_{n=1}^N w^*(u^n)^T L(u^n) w^*(u_n) \\ 
    \text{s.t.} &\ \sigma_{\min}(L(u^n)) 
    \geq \underline{\lambda} \nonumber, \quad \forall n, \\ 
    &\ \sigma_{\max}(L(u^n)) \leq \overline{\lambda}, \quad \forall n. \nonumber
\end{align}
While appealing, this formulation has several drawbacks. From a computational point of view, this is a difficult problem. If $L(u)$ is a linear mapping, the problem is a linear semidefinite problem. However, this only ensures that the model $L(u)$ outputs a positive semidefinite (PSD)  matrix only on the input data, and not necessarily out-of-sample. If $L(u)$ is a constant function, that is we always use the same matrix independently of $u$, this problem does become more tractable. 

Aside from computational complexity, there is another issue with the above. In practice, we will not perform enough iterations to converge, and for the sake of speed, will only use relatively few iterations. While ensuring the rate of convergence is useful, it leaves a lot missing when trying to calculate the regret after a small number of iterations, say only 5 or 10 update iterations. Practically, a more useful problem is to learn $L(u)$ by solving
\begin{equation}
\label{eq:final}
    \min_{L} \sum_{n=1}^N g_{u^n} (\hat{w}_{L(u^n),T}(u^n)).
\end{equation}
Moreover, we can bound the generalization gap of learning ${L}$ from data. Suppose we are given a dataset of $N$ i.i.d. samples. Let $C({L})$ be the expected cost, and $\hat{C}({L})$ the empirical cost:
\begin{align}
    C({L}) =&\ \mathbb{E}_u [g_u(\hat{w}_{{L}(u), T}(u)] \\ 
    \hat{C}({L}) =&\ \frac{1}{N} \sum_{n=1}^N g_{u^n} (\hat{w}_{{L}(u^n),T}(u^n))
\end{align}
Now suppose we are learning the function $L(u)$ from a hypothesis class $\mathscr{L}$. This is a set of functions mapping $u \in \R^d$ to a matrix $L \in \R^{d \times d}$. Essentially, a mapping $\R^d \to \R^{d^2}$. Rademacher complexity aims to define the complexity of this set of functions $\mathscr{L}$.
\begin{definition}[Multidimensional Rademacher Complexity]
 The empirical Rademacher complexity of the hypothesis class of function $\mathscr{L}$ from $\R^d \to \R^{d^2}$  is given by
    \begin{equation}
    \mathscr{R}_N(\mathscr{L}) = \mathbb{E}_{u^1,\dots,u^N} \mathbb{E}_{\sigma}\left[ \sup_{L \in \mathscr{L}} \frac{1}{N} \sum_{n=1}^N \sum_{k=1}^{d^2} \sigma_{nk} L_k(u^n) \right]
\end{equation}
where $\sigma_{nk}$ are i.i.d. variables uniformly sampled from $\{-1,1\}$ (also known as Rademacher random variables). 
\end{definition}
\begin{theorem}[Generalization bound]
\label{thm1}
With probability $1 - \delta$, for any function ${L} \in \mathscr{L}$,
\begin{equation}
    C({L}) \leq \hat{C}({L}) + \lambda_T \cdot \mathscr{R}_N(\mathscr{L}) + \sqrt{\frac{\log(1/\delta)}{N}}
\end{equation} 
where $\lambda_T \leq \sqrt{2} \beta \gamma D \cdot \frac{1 - \paren{1 - \gamma \sigma_{\min} + \alpha }^T}{ \gamma \sigma_{\min} - \alpha}
$, $D$ is the diameter of the feasible region $\mathcal{P}$ and $\sigma_\min$ is the smallest eigenvalue possible of any matrix output from the class $\mathscr{L}$ and where $g_u(w)$ is $\alpha$-smooth and $\beta$-Lipschitz as in assumption \ref{ass:smooth}. Finally, $ \mathscr{R}_N(\mathscr{L}) $ denotes the Rademacher complexity of the hypothesis class $(\mathscr{L})$.
\end{theorem}
For many hypothesis classes $\mathscr{L}$, we can bound $\mathscr{R}_N(\mathscr{L})$ by a term that converges to $0$ as $N\to\infty$ and at a rate $O(1/\sqrt{N})$ for common function classes like linear functions. See for example \cite{bartlett2002rademacher}.

\begin{corollary}[Generalization bound for $T\to\infty$]
As $T$ approaches infinity, the approximation $\hat{w}_{L,T}$ will converge to the optimal solution of the surrogate objective $r_u(w)$. In this case, the generalization bound simplifies to 
\begin{equation}
    C({L}) \leq \hat{C}({L}) + \frac{\sqrt{2}\beta \gamma D}{\gamma \sigma_\min - \alpha} \cdot \mathscr{R}_N(\mathscr{L}) + \sqrt{\frac{\log(1/\delta)}{N}}.
\end{equation}
as long as $1 + \alpha \geq \gamma\sigma_\min \geq \alpha$. 
\end{corollary}
In practice, we can choose both $\gamma$ and $\sigma_\min$. Recall $\sigma_\min$ is the smallest possible eigenvalue of any matrix $L$ that is the output of a function from $\mathscr{L}$. To control $\sigma_\min$ algorithmically, see section \ref{sec:control}. 


\begin{nicebox}
\textbf{Approximation framework.}
 The approximation $\hat{w}(u)$ is generated by $T$ iterations of the update rule given by
 \begin{equation}
     \hat{w}_{t+1} = \pi_{\mathcal{P}} \paren{\hat{w}_{t} - \eta \nabla g_u(w) - \gamma \cdot L(u)w}
 \end{equation}
 for $t = 0, \dots, T-1$, choosing $\hat{w}_0 = 0$, $\pi_{\mathcal{P}}$ the projection operator on the feasible region $\mathcal{P}$ and $L(u)$ a positive semidefinite matrix. We learn such  an $L(u)$ by solving \eqref{eq:final}
 \begin{equation}
 \label{eq:train}
    \min_{L \in \mathscr{L} } \sum_{n=1}^N g_{u^n} (\hat{w}_{L(u^n),T}(u^n)).
 \end{equation}
 where $\mathscr{L}$ defines the set of possible functions $L$. As examples in this paper, we will consider $\mathscr{L}$ to be the set of constant functions (that is, $L(u)$ is a constant matrix independent of $u$), and the set of linear functions of $u$. 
\end{nicebox}

\begin{proof}[Proof of Theorem \ref{thm1}]
 The results of \cite{bartlett2002rademacher} can be applied directly to the composite cost function $g_u(\hat{w}_{L(u),T}(u))$. We can view this as a composition $g_u \circ \hat{w}_{L, T} \circ \mathscr{L} $. Theorem 8 of \cite{bartlett2002rademacher} gives us, with probability $1-\delta$ over $N$ i.i.d. training data samples, that the following inequality holds for all $L \in \mathscr{L}$,
 \begin{equation}
 \label{eq:rad-thm}
 C(L) \leq \hat{C}(L) +  \mathscr{R}_N(g_u \circ \hat{w}_{L, T} \circ \mathscr{L}) + \paren{\frac{8\log 2/\delta}{N}}^{1/2}.
 \end{equation}
 Next, using the vector contraction inequality from \cite{bartlett2002rademacher}, we can further bound the Rademacher complexity by
 \begin{equation}
    \label{eq:r-bound}
     \mathscr{R}_N(g_u \circ \hat{w}_{L, T} \circ \mathscr{L}) \leq \sqrt{2} \lambda \mathscr{R}_N(\mathscr{L})
 \end{equation}
 where the function $g_u\circ \hat{w}_{L,T}$ is $\lambda$-Lipschitz. It remains to bound the lipschitz constant $\lambda$. Since $g_u(w)$ is $\alpha$-Lipschitz with respect to $w$ by assumption, we will first focus on the Lipschitz constant of $\hat{w}_{L,T}$ with respect to $L$. That is, we will show that for matrices $L_1, L_2$, and any $u$,
 \begin{equation}
    \norm{\hat{w}_{L_1, T}(u) - \hat{w}_{L_2, T}(u)} \leq \lambda_T \norm{L_1 - L_2}.
 \end{equation}
 
We do this by induction and first write a recurrence relation defining $\lambda_{T+1}$ in terms of $\lambda_{T}$. We begin by rewriting $\hat{w}_{L,T}(u)$ in terms of $\hat{w}_{L, T-1}(u)$. For ease of notation for the remained of the proof, we will rewrite $\hat{w}^{L_k}_{T} = \hat{w}_{L_k, T}(u), k=1,2$ and assume that $L_1, L_2$ as well as $u$ are fixed. We then have
\begin{equation}
     \norm{\hat{w}^{L_1}_{T+1} - \hat{w}^{L_2}_{T+1}} = \norm{\pi\paren{\hat{w}^{L_1}_{T} - \eta \nabla g_u(\hat{w}^{L_1}_{T}) - \gamma L_1\hat{w}^{L_1}_{T}} - \pi\paren{\hat{w}^{L_2}_{T} - \eta \nabla g_u(\hat{w}^{L_2}_{T}) - \gamma L_2\hat{w}^{L_2}_{T}}}.
\end{equation}
Note that any projection operator is non-expansive. Using this to remove the $\pi$ operator and rearranging terms give us 
\begin{equation}
\norm{\hat{w}^{L_1}_{T+1} - \hat{w}^{L_2}_{T+1}}  \leq \norm{
\paren{\paren{I - \gamma L_1}\hat{w}^{L_1}_{T} - \paren{I - \gamma L_2}\hat{w}^{L_2}_{T} } } 
+ \eta\norm{\nabla g_u(\hat{w}^{L_1}_{T}) - \nabla g_u(\hat{w}^{L_2}_{T})}.
\end{equation}
By assumption of $\alpha$-smoothness, the gradient of $g$ is $\alpha$-Lipschitz with respect to $w$. So, we can bound the right-most term above by 
\begin{align}
    \norm{\nabla g_u(\hat{w}^{L_1}_{T}) - \nabla g_u(\hat{w}^{L_2}_{T})} \leq&\ \alpha \norm{\hat{w}_{T}^{L_1} - \hat{w}_{T}^{L_1}} \\ 
    \leq&\ \alpha \cdot \lambda_{T}\norm{L_1 - L_2}. \label{eq:grad-lip}
\end{align}
We are left to bound the Lipschitz constant of the left term: 
\begin{align}
    \norm{
\paren{I - \gamma L_1}\hat{w}^{L_1}_{T} - \paren{I - \gamma L_2}\hat{w}^{L_2}_{T} } &\  \\ 
\leq &\ \norm{
\paren{\paren{I - \gamma L_1}\hat{w}^{L_1}_{T} - \paren{I - \gamma L_2}\paren{\hat{w}^{L_2}_{T} + \hat{w}^{L_1}_{T} -\hat{w}^{L_1}_{T}   } }} \nonumber \\ 
\leq &\ \norm{\gamma\paren{L_1 - L_2}\hat{w}^{L_1}_{T} + \paren{I - \gamma L_2}\paren{\hat{w}^{L_1}_{T} - \hat{w}^{L_2}_{T}}}  \\
\intertext{We can bound $\hat{w}_{T}^{L_1}$ by the diameter $D$ of the feasible region. Moreover, the minimum eigenvalue across all matrices $L$ is $\sigma_{\min}$. So, we can further bound this as}
\leq &\ \gamma D \norm{L_1 - L_2} + \paren{1 - \gamma \sigma_{\min} }\norm{\hat{w}^{L_1}_{T} - \hat{w}^{L_2}_{T}}\\
\leq &\ \gamma D \norm{L_1 - L_2} + \paren{1 - \gamma \sigma_{\min} }\lambda_T\norm{L_1 - L_2}.
\end{align}
Therefore, combining this with the Lipschitz term from \eqref{eq:grad-lip}, we find 
\begin{equation}
    \lambda_{T+1} \leq \gamma D + \lambda_{T} \cdot \paren{1 - \gamma\sigma_{\min} + \alpha }.
\end{equation}
Furthermore, as a base case $\lambda_0 = 0$ since after $T = 0$ iterations, $\hat{w}^{L_1}_{0} = \hat{w}^{L_2}_{0}$ since we always use the same initialization. Therefore, we can solve the recurrence relation to find 
\begin{equation}
    \lambda_T \leq \gamma D \cdot \frac{1 - \paren{1 - \gamma\sigma_{\min} + \alpha }^T}{ \gamma\sigma_{\min} - \alpha}
\end{equation}
which proves the theorem. 
\end{proof}
}








\subsection{ProjectNet: tractable algorithms}
\label{section:projectnet}

{\color{black}
We propose tractable learning methods of solving the problems introduced in the previous section. This consists of addressing two key points. (1) We present a differentiable method of projection onto linear constraints and (2) we present a method to control the eigenvalues of the matrix $L(u)$. We denote this model as ProjectNet. Finally, (3) we integrate the ProjectNet model into the end-to-end framework.
}



\subsubsection{Ensuring Feasibility}
\label{sec:ensuring feasibility}

{\color{black}
    First, we discuss how to perform the projection operator $\pi_{\mathcal{P}}$ as part of computing $\hat{w}_{r,T}(u)$. Projection itself is a difficult optimization problem given by $\pi(w) = \arg \min_{y \in \mathcal{P}} \norm{w - y}^2$.
}
We resolve this issue by performing an approximate projection as follows. The only requirement is that each individual projection $\pi_j$ can be done by a differentiable method. We only assume that projection onto a single constraint can be done through a differentiable method. We use Dykstra's projection algorithm \citep{Dykstra} that provides a sequence of differentiable steps to approximate the projection. Let $\mathcal{P}_1, \dots, \mathcal{P}_J$ be any $J$ intersecting convex sets that make up the feasible region. For example, $\mathcal{P}_i = \{ w : h_i(w) \leq 0 \}, \mathcal{P}_{j+p_1} = \{ l_j(w) = 0\}$.  We cyclically project onto $\mathcal{P}_1$ through $\mathcal{P}_J$ until we reach some desired accuracy (distance from satisfying both constraints). We define $\pi_j$ as the projection operators onto sets $\mathcal{P}_j$. After $k$ steps of the iterative projection, we reach an approximation $\tilde{\pi}^k$ defined in Algorithm~\ref{projection-alg}. That is, one must be able to compute $\partial \pi_j(w) / \partial(w)$. For example, for linear optimization problem, we let $\mathcal{P}_1 = \{ w : Aw = b  \}, \mathcal{P}_2 = \{ w : w \geq 0 \}$, so that $\mathcal{P} = \mathcal{P}_1 \cap \mathcal{P}_2$. The projections $\pi_1,\pi_2$ can be  evaluated easily as follows:
\begin{alignat}{2}
\label{eq:projections}
    \pi_1(w) =&\ \arg \min_{y : Ay=b} \norm{w - y}_2^2 \\ 
    =&\ \ w - A^T(AA^T)^{-1}(Aw - b) \\ 
    \pi_2(w) =&\ \arg \min_{y \geq 0} \norm{w - y}_2^2  
             =\ ReLu(w).
\end{alignat}
In the case of linear subspaces, the method simplifies to $\tilde{\pi}^k(w) = \pi_2(\pi_1(\dots(\pi_2(\pi_1((w))\dots))$. This is depicted in Fig.~\ref{fig:projection}. For example, in the case of a polyhedral feasible region, 
the sequence of points $w_k$ is guaranteed to converge to a point in $\mathcal{P} = \mathcal{P}_1 \cap \mathcal{P}_2$ at a geometric rate. In particular, \citet{Deutsch1994} show that there exists $\rho < 1, a > 0$ so that for any integer $k$:
\begin{equation}
\norm{\tilde{\pi}^k(w) - \pi(w)}_2 \leq a \cdot \rho^k,
\end{equation}
where $\pi(w)$ is the exact projection of $w$ onto the feasible region $\mathcal{P}$. However, as the sequence of projections converges closer to a vertex, the corresponding gradients $\partial \tilde{\pi}^k(w) / \partial w$ will also approach zero. Indeed, in the limit, if we have exact projections onto vertices of the feasible polytope, then the gradient is zero. This suggests that one must be careful when choosing the number of iterations $k$ so that they are large enough to provide good approximations, but at the same time consider the trade-off in keeping the corresponding gradients from becoming too small.

\begin{figure}
    \centering
  \includegraphics[scale=0.5]{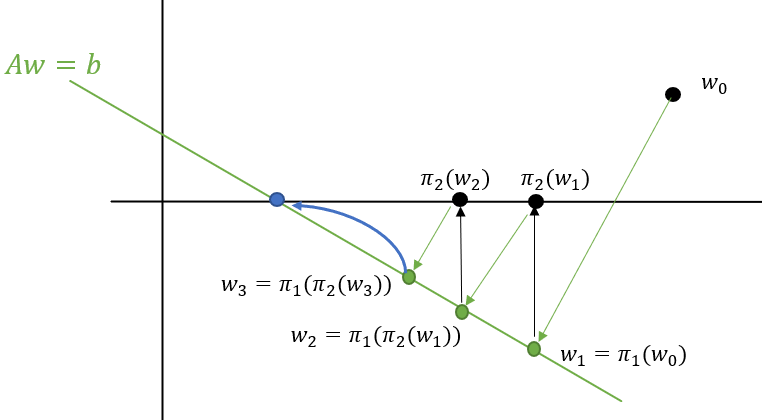}
      \caption{Iterative projection method. The sequence of projections converges to a feasible solution (depicted as the blue point).}
\label{fig:projection}
\end{figure}

\begin{algorithm}
\caption{Dykstra's Projection Method}
\label{projection-alg}
\begin{algorithmic}[1]
\Function{$\tilde{\pi}^k$}{$w$}
\State Initialize $w_J^{0} = w$, and $p_1^{0} = \dots = p_J^{0}=0$. 
\For{$t = 1, \dots, k$}
    \State $w_0^{t} = w_{d}^{t-1}$
    \For{$j = 1, \dots, J$}
        \State $w_j^{t} = \pi_j(w_{j-1}^t + z_j^{t-1})$
        \State $z_j^{t} = w_{j-1}^t + z_{j}^{t-1} - w_{j}^k$
    \EndFor
\EndFor
\State \Return $w_J^{k}$
\EndFunction
\end{algorithmic}
\end{algorithm}

\subsubsection{Controlling eigenvalues}
\label{sec:control}

{\color{black} We now focus on controlling the eigenvalues of $L(u)$. We do so as follows. We propose any auxiliary neural network model which outputs two quantities, an upper triangular matrix $M(u)$ and a diagonal matrix $D(u)$. Then, 
\begin{enumerate}
    \item  The resulting matrix $M(u)M(u)^T$ is symmetric and invertible as long as the diagonal entries of $M(u)$ are positive. 
    \item The matrix $(M(u)M(u)^T) D(u) (M(u)M(u)^T)^{-1}$ exists and moreover its eigenvalues are equal to the diagonal entries of $D(u)$.
    \item To control all eigenvalues to be between $ \underline{\lambda}$ and $\overline{\lambda}$, we can apply the transformation $ \rho(\cdot) $ is the sigmoid activation function appropriately scaled and translated. So, we can set 
    \begin{equation}
        L(u) = (M(u)M(u)^T) \rho(D(u)) (M(u)M(u)^T)^{-1}.
    \end{equation}
    where again $M(u), D(u)$ can be generated from any choice of neural net architecture. The inverse operation is differentiable, and implemented in most existing software. 
\end{enumerate}
    
}

\subsubsection{Model Architecture} 
\label{sec:model}
{\color{black}
Recall from section \ref{sec:analytical} our goal is to learn a function $L(u)$ which outputs an update rule by solving 
\begin{equation}
    \min_{L} \sum_{n=1}^N g_{u^n}(\hat{w}_{L(u^n)}(u^n)).
\end{equation}
To solve this, we make two approximations to make the formulation tractable. Instead of using $\hat{w}_{L}$, the point of convergence, we instead use $\hat{w}_{L, T}$ for some choice of $T$ iterations. Moreover, instead of using the exact projection $\pi$ when computing a single step of the iterative process, we use $\tilde{\pi}_{\mathcal{P}}$ from section \ref{sec:ensuring feasibility} instead. This can be seen in the architecture of the ProjectNet in figure \ref{fig:model}.
}

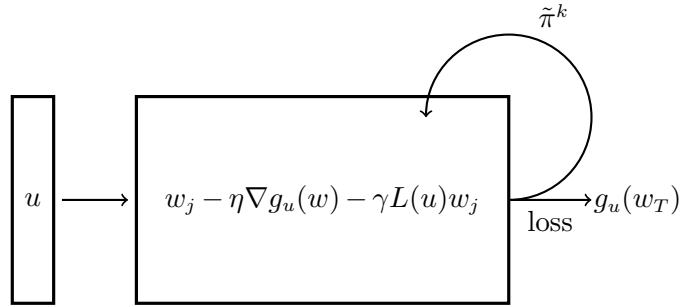
\begin{figure}[b]
  \centering
  \begin{tikzpicture}[scale=0.55]
    \draw[very thick] (0,0) rectangle (1,5) node[pos=.5] {$u$};
    \draw[thick, ->] (1.2 , 2.5) -- (2.8,2.5);
    \draw[very thick] (3,0) rectangle (12,5) node[pos=.5] 
    {$ w_j - \eta \nabla g_u(w) - \gamma L (u)w_j$};
    \draw[thick, ->] (12,2.5) arc (-90:180:2) node[above=3.5em,right=3.5em] {$\tilde{\pi}^k$};
    \draw[thick, ->] (12, 2.5) -- (14, 2.5) node[midway, below] {loss};
    \draw[draw=none] (14.5,2) rectangle (15.5,3) node[pos=.5] { $\ g_u(w_T)$};
\end{tikzpicture}    
    \caption{ProjectNet architecture}
    \label{fig:model}
\end{figure}

\paragraph{Improvement over gradient descent.} We now show the runtime improvements of our proposed approach compared to traditional projected gradient descent. We show additional computational results for a different non-linear problem in section \ref{sec:electricity}. For ease of clarity, we add the full details of the setup in Appendix \ref{app:matching}. We present computational results comparing the two approaches on a maximum matching problem with $n=50$ nodes and $n^2 = 2,500$ edges/variables.

We also train a ProjectNet model with $T_0 = 5$ iterations, and compare the objective value of its solution for iterations up to $T_1 = 35$ on testing data. See figure~\ref{fig:projectnet vs gd}. In particular, we measure the average relative regret of decisions. That is, given realized edge weights $u$ and decision $\hat{w}$, the relative regret is the percent difference in objective between the objective of $\hat{w}$ and the optimal decision in hindsight:
  $  (g_u(\hat{w}) - g_u(w^*(u)) / g_u(w^*(u) $.
We see that indeed the ProjectNet method improves consistently in accuracy as the number of iterations $T_1$ is extended from the $T_0$ steps that were used during training.

Note that when compared to the traditional gradient descent approach, the ProjectNet approach performs better using fewer iterations.
It maintains this edge even for steps $T_1 > T_0$, which it has not trained upon. However for larger $T_1$ traditional gradient descent is better able to converge to the optimal solution and it overtakes the ProjectNet method. But for the end-to-end framework it is beneficial to use a smaller number of iterations $T_1$ since this is computationally more efficient, and keeps the gradient $\nabla \hat{w}(u)$ from approaching zero. In the regime of smaller $T_1$, the ProjectNet method also has an advantage in terms of objective function value, with up to $12.5\%$ improvement. See Figure~\ref{fig:improvement}. 

\begin{figure}[htb]
    \centering
\begin{subfigure}{.45\textwidth}
  \centering
    \includegraphics[scale=0.45]{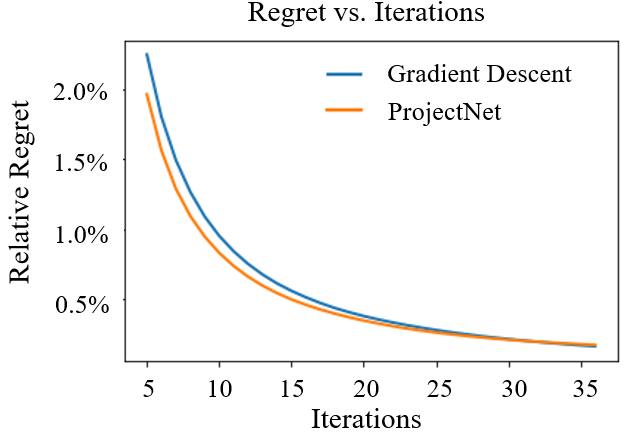}
    \caption[size=11pt]{Regret of ProjectNet compared to gradient descent as iterations $T$ increase. 
    }
    \label{fig:projectnet vs gd}
\end{subfigure}\hfill
\begin{subfigure}{.45\textwidth}
    \includegraphics[scale=0.4]{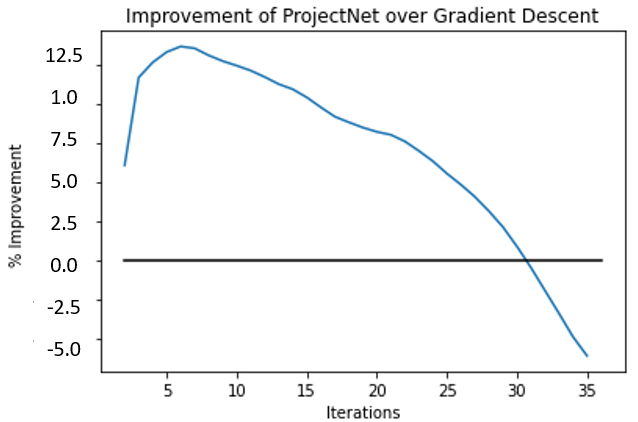}
    \caption{Percent improvement in relative regret of the ProjectNet model compared to gradient descent. 
    }
    \label{fig:improvement}
\end{subfigure}
\caption{
Comparison of ProjectNet to Gradient Descent }
\label{fig:matching}
\end{figure}

\subsection{End-to-End Learning via ProjectNet}

First, ProjectNet is trained to learn solutions to the optimization problem $w^*(u)$. In particular, given some cost vectors $u^1, \dots, u^N$, we aim to learn some $\hat{w}$ parametrized by the update linear layer $L$ which minimizes empirical cost as introduced in \eqref{eq:train}. 

Note that we never need to solve the nominal optimization problem $w^*(u^n)$ during this process. In addition, we may use any data $u^1, \dots, u^n$ that we wish, not necessarily only the vectors from the original training data of $(x^n, u^n)$. We can train using, say, $T_0$ iterations of the recurrent network. The entire end-to-end method to learn forecasts $f_\theta(\cdot)$ is now described in Algorithm~\ref{alg} which follows the steps in diagram \ref{fig:deterministc_diag} shown earlier. In short, we make a forecast $f_\theta(x^n)$ and differentiate through the approximate corresponding solution $\hat{w}(f_\theta(x^n))$ to update $\theta$. During the training step of ProjectNet, we may use $T_0$ iterations, while when subsequently evaluating $\hat{w}(f_\theta(x_n))$, we may use any $T_1 \geq T_0$ iterations to improve the accuracy of the ProjectNet's approximation. 
\begin{algorithm}[bt]
\caption{End-to-End Learning via ProjectNet}
\label{alg}
\begin{algorithmic}
\Function{ProjectNet End-to-End}{$(x^1, u^1), \dots, (x^N, u^N)$}
\State $\hat{w}(\cdot) \leftarrow$ \textproc{TrainProjectNet}()
\State Initialize $\theta$ at random. 
\For{each epoch}
    \For{$n = 1, \dots, N$}
        \State Compute $\nabla_\theta g_{u^n}( \hat{w}(f_\theta(x_N)))$.
        \State Update $\theta$ by any gradient method.
    \EndFor
\EndFor
\State \Return $\theta$
\EndFunction
\end{algorithmic}
\end{algorithm}

\section{Extension to Two-Stage Stochastic Optimization}

\label{section:stochastic-programming}
\label{section:2-stage}

We now consider two-stage linear stochastic optimization problems.  Let $w$ denote the first-stage decision and let $V(w,u)$ denote the second-stage cost of decision $w$ under the realization $u$ of uncertainty. As an example, we can consider a multi-warehouse cross-fulfillment newsvendor problem. The first-stage decision $w$ is the amount of product to allocate to each warehouse. After the decision is made, the demand $u$ is realized, and finally one must fulfill the demand using the initial stocking decision. This would correspond to $V(w,u)$ being a form of a matching problem (which warehouse should fulfill which client). For details and computational results, see section~\ref{sec:cross-fulfillment}.

In general we assume the first stage problem can be formulated as 


\begin{equation}
\begin{array}{lll}
    w^*(D) = & \arg\min_{w} &\ c^Tw + \mathbb{E}_{u\sim D}\left[ V(w,u) \right]  \\ \\
& \text{subject to}
        &\ Aw = b \\ \\
        & &\  w \geq 0. 
\end{array}
\end{equation}


\noindent As for the second stage problem, we allow the uncertainty to impact either the objective or the constraints, and these could depend on the decision $w$ taken in the first stage as well. The problem takes the general form as below, where $w$ is the second-stage decision variable, $T(w,u)$ is some matrix which depends on the first-stage decision $w$ and on the realization $u$:
\begin{equation}
\begin{array}{lll}
    V(w,u) =& \min_{v} &\ d(w,u)^Tv \\ \\ 
& \text{subject to}
        &\ T(w,u) \begin{bmatrix}w \\ v\end{bmatrix}   =  h(w, u)\\ \\
        && v \geq 0.
\end{array}
\label{eq:2ndstage}
\end{equation}

\begin{assumption_}
We assume the model has relatively complete recourse.\end{assumption_}
That is, any feasible $w$ in the first stage leads to a feasible second stage problem for any uncertainty realization.
This assumption is often satisfied, for example, this is the case in the traditional and cross-fulfilment newsvendor problem (see for example \citet{recourse1}, \citet{birge2011introduction}).


Suppose we observe features $x^n$ and corresponding realizations of uncertainty $u^n$, for $N$ data points, $n = 1, \dots, N$. Given a point forecast $f_\theta(x^n)$, the corresponding decision is $w^*(f_\theta(x^n))$. Afterwards, the value of the uncertainty $u^n$ is realized and we can determine the cost to be $Z(w^*(f_\theta(x^n)), u^n)$, where $Z(w,u) := c^Tw + V(w,u)$. Hence, the cost minimization problem to learn $\theta$ is as follows
\begin{equation}
    \min_\theta \sum_{n=1}^N Z(w^*(f_\theta(x_n)), u_n) = \min_\theta \sum_{n=1}^N c^T(w^*(f_\theta(x_n))) + V(w^*(f_\theta(x_n)), u_n).
\end{equation}





An additional difficulty in solving this problem is that there are nested optimization problems. That is, computing $w^*(f_\theta(x))$ and passing its solution as input to problem $V$. 
Therefore, to further simplify this problem, consider making a first-stage decision $w$ directly from data by some $q_\vartheta(x)$ instead of making an intermediate forecast. Note that this problem differs from the learning problem in the previous section. Here, the goal is to learn a decision rule $q_\vartheta(x)$, whereas previously we made intermediate forecasts { $f_{\theta}(x)$, which was then used to solve the single stage optimization problem of interest to obtain $w^*(f_{\theta}(x))$}. 
We then have
\begin{equation}
\label{eq:Z}
     \min_\vartheta \sum_{n=1}^N c^Tq_\vartheta(x_n)  + V(q_\vartheta(x_n), u_n).
\end{equation}
The difficulty now lies in taking the gradient $\nabla_w Z(w,u) = c + \nabla_w V(w,u)$, since $V$ is a complex optimization problem. Therefore, we propose to use ProjectNet to learn this $V(w,u)$. Specifically, it learns the second-stage decisions $\hat{v}$ in problem (\ref{eq:2ndstage}). Then, we can approximate $V(w,u)$ by $d(w,u)^T\hat{v}$. Notice that this problem of approximating $\hat{v}$, is the same as the ProjectNet problem described for the deterministic problem in \ref{nominal-task}. 
Finally, we must ensure that the decisions $q_\vartheta(x)$ satisfy first-stage feasibility constraints $Aq_\vartheta(x_n) = b, q_\vartheta(x_n) \geq 0$. This may be accomplished by applying our approximate projection method.  

The primary difference in this case compared to the initial single-stage problem we presented in section \ref{sec:end-to-end} is that the uncertainty can lie in the constraints. However, this does not pose any problems, as the output of our proposed ProjectNet architecture is still differentiable with respect to parameters in the constraints. First, let us rewrite the update function of the ProjectNet for the second-stage problem. In this case, we aim to learn the optimal second-stage variables $v$ and recall the objective function is  linear $d(w,u)^Tv$, so the gradient is always $d(w,u)$. Finally, we have
\begin{equation*}
    v_{j+1} = \tilde{\pi}^k \left(v_j - \eta d(w,u) - \gamma L (v_j) \right)
\end{equation*}
and we notice that the constraints only appear in the approximate projection $\tilde{\pi}^k$. From Algorithm~\ref{projection-alg} and Equations~(\ref{eq:projections}) one can see $\tilde{\pi}^k$ is a sequence of steps differentiable with respect to the constraints. Hence, the gradient of the approximation $d(w,u)^Tv_j$ is always well-defined.



\paragraph{Point Forecasts vs. Distributional Forecasts} In the same spirit as for single-stage problems, it is sufficient to make point forecasts with an end-to-end approach when the objective $Z(\cdot,\cdot)$ is a loss-type function:


\begin{prop}
\label{prop:point-forecast}
Consider a two-stage stochastic optimization problem as described in section \ref{section:stochastic-programming}, with loss-type objective function $Z(w, u)$. That is, $Z(w,w) = 0$ for feasible $w$. Then, for any distributional forecast, there exists a single point forecast that produces the same solution. In other words, for any distribution $D$, there exists a single point forecast $d$ so that 
\begin{equation}
    \arg \min_{w} \E_{u \sim D} [ Z(w,u) ] = \arg \min_{w} Z(w,d).
\end{equation}
\end{prop}

\begin{proof}
Let $w^*$ be the solution to the problem using distributional forecast $D$:
\begin{equation}
    w^* = \arg \min_{w} \E_{u \sim D} [ Z(w,u) ].
\end{equation}
Now consider making a forecast of exactly $d = w^*$. Then, $w^* = \arg \min_{w} Z(w,d)$, since $Z(\cdot, \cdot)$ is a loss function (i.e., the minimum is achieved at $Z(w,w)$). 
\end{proof}

\section{Computational Results}
\label{section:experiments}

{\color{black}
In this section we present computational results illustrating that the ProjectNet method introduced in this paper is effective in several end-to-end learning settings. We show this on several tasks: (1) a two-stage multi-warehouse cross-fulfillment newsvendor problem in which the first stage consists of allocating supply to many warehouses, and the second consists of optimally fulfilling the realized demand, (2) a real-world electricity planning problem (3) a shortest path problem in which the forecasting step is a computer vision task of predicting edge costs from terrain maps. Moreover, we include additional synthetic experiments on several variations of the multi-item newsvendor problem having linear or quadratic costs and capacity constraints in Appendix \ref{appendix:newsvendor}.
}

\subsection{Multi-warehouse cross-fulfillment newsvendor}
\label{sec:cross-fulfillment}

As an illustration of a two-stage stochastic optimization problem, we consider a multi-warehouse newsvendor problem with cross-fulfillment. 
We consider a setting of $n$ warehouses and $m$ clients with unknown future demand. In the first stage, one must decide on the amount of product to allocate to each individual warehouse. In the second stage, the demand at each client is realized and one must determine the optimal plan to fulfil the demand given the decisions made. In particular, there are traveling unit costs $c_{ij}$ to transport a unit of product from warehouse $i$ to client $j$. For every unit of unmet demand at client $j$, there is a backorder unit cost of $b_j$ and for every unit of product left unused at a warehouse $i$ there is a holding unit cost of $h_i$.

Let $w_i$ denote the first-stage decision of amount of product to allocate at warehouse $i$ and $V(w, d)$ the minimum cost of fulfilling a demand of $d = (d_1, \dots, d_m)$ for the clients. 
\begin{equation}
\begin{array}{lll}
     V(w, d) =& \min_{v \geq 0} &\ \sum_{i=1}^n\sum_{j=1}^m c_{ij} v_{ij} + \sum_{j=1}^m b_j \paren{d_j - \sum_{i=1}^n v_{ij}}^+ + \sum_{i=1}^n h_i \paren{\sum_{i=1}^n v_{ij} - s_i}^+ \\ \\
& \text{subject to}
        &\ \sum_{i=1}^n v_{ij} \leq w_i
\end{array}
\end{equation}
Finally, there is a unit cost of $c$ of allocating a single product to any warehouse. Hence, the cost of the first and second stages when making decision $w$ against future realization of $d$ is
\begin{equation}
    Z(w,d) = c \cdot \sum_{i=1}^n w_i + V(w, d)
\end{equation}
We assume we are given data $(x^1, d^1), \dots, (x^N, d^N)$ consisting of observed features $x^n$ and corresponding demand realization of $d^n$. We generate this data as follows. Each $x^n$ is drawn from a normal gaussian distribution, and $d^n$ is given by a deterministic quadratic function of $x^n$. In particular, $(d^n)_j = (q^Tx^n)^2_j$ for a fixed vector $q$ which is initially generated at random.
To learn the decision rule $f_\theta(x)$, we solve
\begin{equation}
    \min_\theta \sum_{n=1}^N Z(f_\theta(x^n), d^n).
\end{equation}
The primary difficulty lies in calculating $V(w,d)$ and the gradient $\partial V(w,d) / \partial w$. We use ProjectNet to approximate these.

\begin{table}[t]
    \centering
\begin{tabular}{c c c c}
    \toprule
     \# locations & Predict-then-Optimize & ProjectNet & OptNet \\ \midrule
     20 & 7.58 & 2.94 / 27s & 2.93 / 60s \\ 
     40 & 9.95 & 3.54 / 88.6s & 3.52 / 1200s \\
     \bottomrule
\end{tabular}
    \caption{Cross-fulfilment newsvendor results. Recording average cost on test set and average running time per epoch.}
    \label{tab:cf}
\end{table}

\begin{figure}[b]
    \centering
    \includegraphics[scale=0.5]{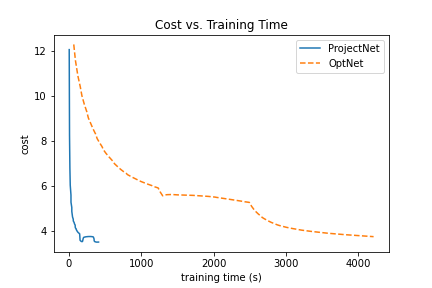}
    \caption{Task-based cost as a function of training time in the cross-fulfillment problem with 40 nodes. 
    }
    \label{fig:cf-runtimes}
\end{figure}

We  compare with the traditional predict then optimize method which separates the prediction and optimization, as well as the end-to-end OptNet method \citep{optnet}. For the predict-then-optimize method we simply learn a forecasting model which aims to minimize the mean-squared error. Since the demand distribution is a deterministic function of the features, this method alone would retrieve the optimal solution given infinite data. Given there will inherently be some error in the model, this will not happen and we see in the experiments that the end-to-end methods greatly outperform this approach. We adopt a similar decision rule idea for two-stage decision problems that we proposed in Section~\ref{section:2-stage} for the ProjectNet method in the case of the OptNet method. We use the following formulation to calculate the fulfilment cost in the OptNet method after the allocation decision $w$ has been made and the demand $d$ is realized:
\begin{equation}
\begin{array}{lll}
     V_{\text{optnet}}(w, d) =& \min_{v,q,p \geq 0} &\ \sum_{i=1}^n\sum_{j=1}^m c_{ij} v_{ij} + \sum_{j=1}^m b_j q_j  + \sum_{i=1}^n h_i p_j + \alpha (\norm{v}^2 + \norm{p}^2 + \norm{q}^2) \\ \\ 
& \text{subject to}
        &\ q_j \geq d_j - \sum_{i=1}^n v_{ij} , \quad  p_j \geq \sum_{i=1}^n v_{ij} - s_i , \quad  \sum_{i=1}^n v_{ij} \leq w_i 
\end{array}
\end{equation}
where again we introduce variables $p,q$ to denote the amount of overstocked or understocked units in order to linearize the objective, respectively. Again, we choose the regularization term $\alpha = 0.01$.

In Table~\ref{tab:cf} we observe that both end-to-end methods clearly outperform the predict then optimize baseline in terms of accuracy. Yet again in this setting we observe a significant decrease in the training time for ProjectNet over Optnet, running more than twice as fast for a 20-location problem. As the problem size grows to double (40 locations), the running time of OptNet increases significantly, nearly twenty times, while the ProjectNet method only increased threefold. For the 40 location example, we see in Figure~\ref{fig:cf-runtimes} the comparison of the cost of the decisions made by each approach as a function of the training time.


{\color{black}
\subsection{Electricity planning}
\label{sec:electricity}

We now consider an electricity generation and planning problem using data from PJM, an electricity routing company coordinating
the movement of electricity throughout 13 states. Our objective is to plan electricity generation over the next 24 hours of the data. The operator incurs a unit cost $\gamma_e$ for excess generation and a cost $\gamma_s$ for shortages. The cost of generating $w_1, \dots, w_{24}$ while true demand is $u_1, \dots, u_{24}$ is given by
$
    g_u(w) = \sum_{i=1}^{24} \gamma_s \max\{ u_i - w_i, 0\} + \gamma_e \max\{ w_i - u_i, 0 \} + 1/2 (w_i-u_i)^2$. Moreover, there are additional ramp-up constraints, that the  generation from one hour to the next cannot differ by more than $r = 0.4$. The constraints are given by $ |w_{i+1} - w_{i}| \leq r, \ i = 1, \dots 23$ and $w_i \geq 0$.
\begin{figure}[t]
    \centering
\begin{subfigure}{.33\textwidth}
    \includegraphics[scale=0.33]{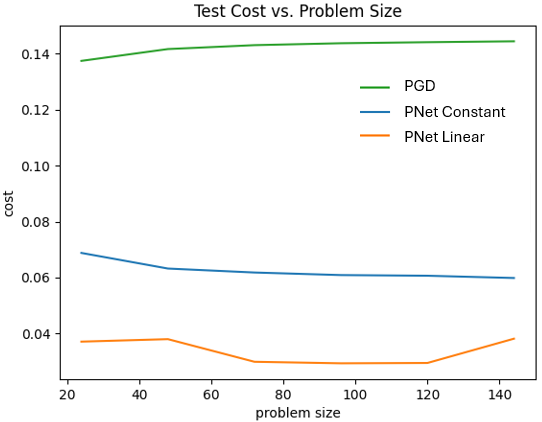}
    \caption{Problem size cost comparison. 
    }
    \label{fig:prob-size}
\end{subfigure}\hfill
\begin{subfigure}{.66\textwidth}
    \centering
    \begin{tabular}{c c c c c}
        \toprule
        Size & ProjectNet & ProjectNet Linear & CVXPY & OptNet \\ \midrule
        24 & 0.12 & 0.09 & 2.32 & 1.54 \\ 
        48 & 0.103 & 0.11 & 2.64 & 3.14 \\ 
        72 & 0.108 & 0.164 & 2.59 & 8.41 \\ 
        96 & 0.11 & 0.33 & 2.59 & 19.06 \\ 
        120 & 0.10 & 0.18 & 2.67 & 21.78 \\ 
        144 & 0.07 & 0.27 & 3.19 & 35.47 \\ \bottomrule
    \end{tabular}
    \caption{Running time (in seconds) of each approach with increasing problem size.}
    \label{tab:runtime}
\end{subfigure}
\caption{
ProjecetNet accuracy and runtime on electricity scheduling. 
}
\label{fig:electricity-res}
\end{figure}

\paragraph{ProjectNet accuracy and runtime} We first focus on the accuracy of the ProjectNet model to approximate the optimization problem as well as the runtime required. We compare against traditional projected gradient descent  (PGD). In addition, we consider two versions of ProjectNet, one where $L(u)$ is a constant function (using the same matrix $L$ for all $u$, we denote this PNet Constant in the table) and where $L(u)$ is a linear function of $u$ (we denote this as ProjectNet Linear in the table). We also compare against the runtime of OptNet and the CVXPY layer developed in \cite{agrawal2019differentiable}. We perform experiments along multiple axes. First, as we increase the amount of training data available, and second as we increase the size of the optimization problem. We increase the size by increasing the planning horizon from one day up to 5 days (hence, having to solve an optimization problem from 24 to 120 variables). All methods (our two versions of ProjectNet, Projected Gradient Descent (PGD), OptNet, and the CVXPY approach) will run 10 update iterations for its approximations.

In figure \ref{fig:data} we see the effect of the amount of data on the accuracy of solutions generated by ProjectNet on the test data from the electricity scheduling problem. As expected, accuracy improves as data increases, and ultimately begins to plateau. We present results for both learning a constant matrix $L$ (PNet Constant) and learning a function $L(u)$ (PNet Linear). Changing the matrix $L$ depending on $u$ gives the model more flexibility to learn better approximations and has over 50\% lower cost than using a constant $L$.

\begin{figure}[t]
    \centering
    \includegraphics[width=0.45\linewidth]{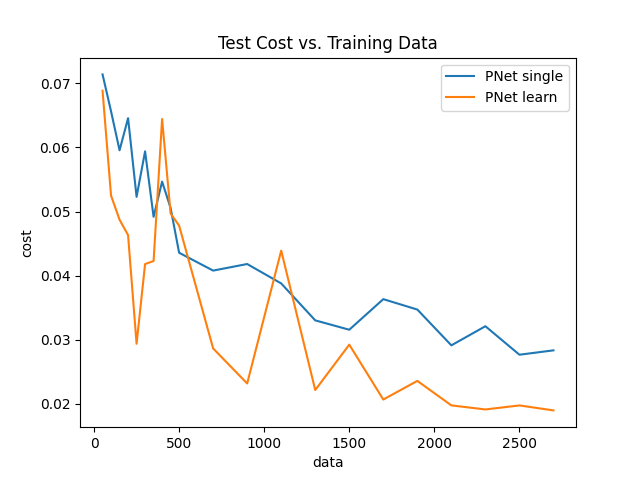}
    \caption{Average cost of approximate solutions for both approaches of ProjectNet, with constant $L$ and linear function $L(u)$.}
    \label{fig:data}
\end{figure}

\begin{figure}[b]
    \centering
\begin{subfigure}{.5\textwidth}
    \includegraphics[width=0.75\linewidth]{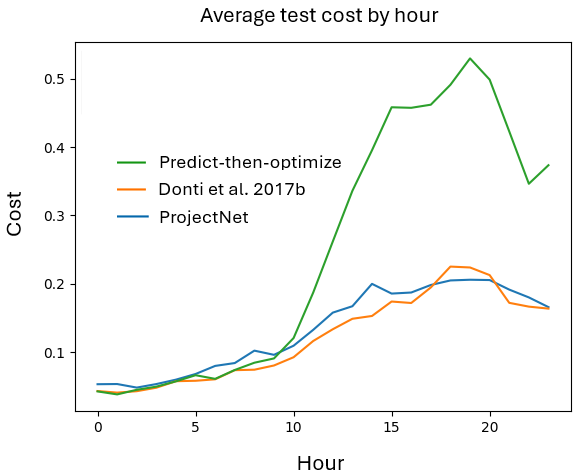} 
    \caption{Average test cost by hour on electricity problem. 
    }
    \label{fig:hour-comparison}
\end{subfigure}\hfill
\begin{subfigure}{.5\textwidth}
    \centering
    \begin{tabular}{c c}
        \toprule
        & Runtime per epoch \\ \midrule
        ProjectNet &  1.45 (s) \\ 
        \cite{donti2017task} & 5.33 (s) \\ \bottomrule
    \end{tabular}
    \caption{Running time (in seconds) of each approach with increasing problem size.}
    \label{tab:runtime-elec}
\end{subfigure}
\caption{
ProjecetNet accuracy and runtime on electricity scheduling. 
}
\label{fig:electricity-end-to-end}
\end{figure}

\paragraph{End-to-end results} Finally, we implement the ProjectNet into the end-to-end framework to train a model to make predictions and corresponding decisions. We use a two-layer (each layer of width 200) network with an additional residual connection from the input to the output layer. In addition, we also perform data augmentation, creating new features like non-linear functions of the temperature, one-hot-encodings of holidays and weekends, and yearly sinusoidal features. The setup is similar to that used in \cite{Donti2017}. The same model and data are used for all models. Hyperparameters are chosen to be the same as parameters chosen from the original paper that introduced the dataset and the benchmark method. We show the average cost incurred by each method during each hour of the day during the last year of data which we use as testing data. See Figure \ref{fig:hour-comparison}. While the ProjectNet-based method incurs a small increase in cost (less than 5\%) it is significantly faster to train as seen in Table \ref{tab:runtime} and \ref{tab:runtime-elec}. 
}



\subsection{Warcraft Shortest Path}

We use the Warcraft II tile dataset \citep{warcraft} which was first introduced in \citep{Pogancic2020} to test their end-to-end approach for combinatorial problems 
On this dataset, we compare our end-to-end method against their method as well as with a traditional two step predict then optimize method. The task consists of predicting costs of travelling over a terrain map and subsequently determining the shortest path between two points. In particular, each datapoint consists of a terrain map defined by a $12 \times 12$ grid where each vertex represents the terrain with a fixed unknown cost. The forecasting aspect is to determine the vertex weights given such an image, and the optimization aspect is to determine the shortest path from the top left to bottom right vertices. See figure \ref{fig:projectnet path} (top left) for a sample of terrain tiles.

\begin{figure}[b]
    \centering
    \includegraphics[scale=0.4]{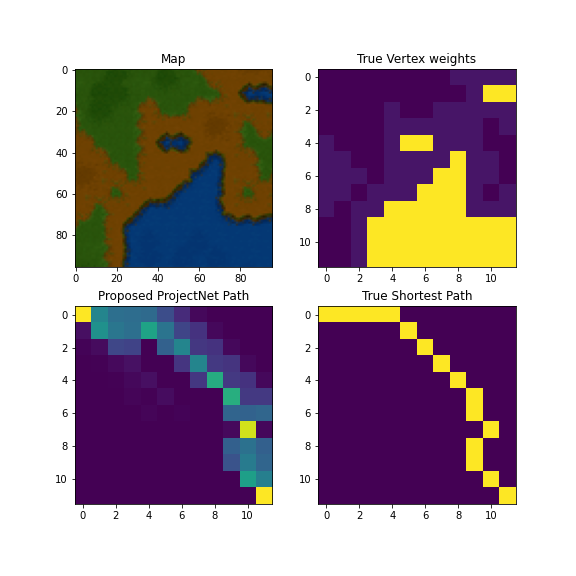}
    \caption{Sample terrain and path proposed by ProjectNet.}
    \label{fig:projectnet path}
\end{figure}

The nominal shortest path problem can be formulated as follows, where $w_{i,j}$ is a variable deciding if edge from node $i$ to node $j$ should be chosen, $u_{i,j}$ is the cost of choosing edge from node $i$ to $j$, and for simplicity $O(i)$ is the set of edges leaving node $i$ and $I(i)$ is the set of incoming edges into node $i$. Finally, the path begins at node $a$ and ends at node $b$:
\begin{equation}
\begin{array}{lll}
     \ w^*(u) = &
    \arg \min_{i,j} & u_{i,j} w_{i,j}  \\ \\
& \text{subject to}
        & \sum_{j \in I(i)} w_{j,i} - \sum_{j \in O(i)} w_{i,j} = 0, \forall i \not= a,b \\ \\
  &      & \sum_{j \in O(a)} w_{a,i} = 1 , \quad  \sum_{i \in I(b)} w_{i,b} = 1 
\end{array}
\end{equation}
Figure~\ref{fig:projectnet path} (bottom left) illustrates an example of the path learned by the ProjectNet method, with the bottom right figure illustrating the true shortest path given perfect knowledge of vertex weights. 

The forecasting task is much more complex in this example than the previous ones. In particular, it is a computer vision problem to learn vertex weights given an image. It is always important to choose the right forecasting model dependent on the problem at hand. A widely used model for computer vision is the so-called residual network \citep{he2016deep}. Traditionally, this is a deep network with 18 layers. 
In contrast, as in \citet{Pogancic2020}, we used only the first 5 layers of the architecture's structure for all experiments. The baseline model is a two-stage predict-then-optimize method which first trains a model by training on minimizing mean-squared error in predicting edge weights, then independently optimizing to find the shortest path. This method performs significantly worse than the end-to-end approaches. Hyperparameters are chosen to be the same as parameters chosen from the original paper that introduced the dataset and the benchmark method.


Comparing as in \citet{Pogancic2020}, we report the percentage of test instances for which various methods found an optimal path in Table~\ref{tab:warcraft results}. We can see the ProjectNet method's accuracy is competitive and more crucially, the running time of our approach is 19\% faster than the end-to-end method of \citep{Pogancic2020}. 

\begin{table}[H]
    \centering
    \caption{Percentage of testing data for which optimal path was found on the warcraft shortest path problem. Runtime reports average running time in seconds per epoch. \vspace{1em} }
    \label{tab:warcraft results}
    \begin{tabular}{lcc}
        \toprule
        Method & Matches & Runtime \\ \midrule 
        ResNet Baseline & 40.2\% & 9.2s \\ 
        \citet{Pogancic2020} & 86.6\% & 81.3s \\ \midrule
        ProjectNet & 83.0\% & 68.3s  \\
            \bottomrule
    \end{tabular}   
\end{table}

{\bf Conclusions}
In this paper we studied the optimization under uncertainty problem. 
The traditional approach for tackling such a problem with uncertainty is the predict then optimize approach (e.g., first perform the prediction tasks, and then use these forecasts as inputs for  downstream  optimization problem). 
Rather in this paper we proposed a tractable end-to-end learning approach. We introduced a novel method to solve the end-to-end learning problem by introducing a novel neural network based method for meta-optimization. 
The proposed approach learns to approximately solve an easier underlying optimization problem. We established analytical results that justify our modelling choices.
Furthermore, we applied this end-to-end learning approach to various supply chain, electricity scheduling, shortest path and maximum matching problems.  We have shown in computational experiments that the ProjectNet method is computationally more efficient than other end-to-end methods while still being competitive in terms of task-based loss against other existing end-to-end methods.

\bibliographystyle{plainnat} 
\bibliography{bibfile} 





\newpage 
\begin{APPENDICES}
\ECHead{\centering{\underline{Appendix}}}

\section{Synthetic experiments}
\label{app:exp}

\subsection{Maximum Matching}
\label{app:matching}
In what follows, we aim to show the improvement of this approach over using a traditional projected gradient method. We present computational results comparing the two approaches on a maximum matching problem. We consider a fully-connected bipartite graph with $n$ nodes in each part, and values $u_{ij}$ assigned to the edge connecting nodes $i$ and $j$ from opposite parts. For the experiment, we use $n = 50$, inducing an optimization problem with $n^2=2,500$ edges/variables. The maximum matching problem is given by
\begin{equation}
\begin{array}{lll}
w^*(u) =& \arg \max_{w} & \sum_{i,j} u_{ij}w_{ij} \\ \\
&\text{subject to} & \sum_{j=1}^n w_{ij} \leq 1, \quad \forall i = 1,\dots,n \\ \\ 
& & \sum_{i=1}^n w_{ij} \leq 1, \quad \forall j = 1,\dots,n \\ \\
&                 & 0 \leq w \leq 1,
\end{array}
\end{equation}
where $u_{ij}$ describes the value of choosing edge from node $i$ to $j$, and $w_{ij}$ represents the variable that decides whether to choose edge $(i,j)$. These variables can be viewed as flow, and the constraints ensure the flow out of a node, or into a node, is at most 1. At optimality, the solution is guaranteed to be integer, and hence equivalent to choosing a single edge.

Finally, note that the formulation has inequality constraints, however our framework was specified using only equality and nonnegativity constraints.  In general, we can transform any problem with inequality constraints $Aw \leq b$ into one with equality constraints as in (\ref{nominal-task}) by adding slack variables $s$: $Aw + Is = b, s \geq 0$. 
We define the projected gradient descent sequence of points $w_{t+1} = \pi(w_t + \eta \cdot u)$, for edge weights $u$.
We also train a ProjectNet model with $T_0 = 5$ iterations, and compare the objective value of its solution for iterations up to $T_1 = 35$ on testing data. See figure~\ref{fig:projectnet vs gd app}. In particular, we measure the average relative regret of decisions. That is, given realized edge weights $u$ and decision $\hat{w}$, the relative regret is the percent difference in objective between the objective of $\hat{w}$ and the optimal decision in hindsight:
  $  (g_u(\hat{w}) - g_u(w^*(u)) / g_u(w^*(u) $.
We see that indeed the ProjectNet method improves consistently in accuracy as the number of iterations $T_1$ is extended from the $T_0$ steps that were used during training.

Note that when compared to the traditional gradient descent approach, the ProjectNet approach performs better using fewer iterations.
It maintains this edge even for steps $T_1 > T_0$, which it has not trained upon. However for larger $T_1$ traditional gradient descent is better able to converge to the optimal solution and it overtakes the ProjectNet method. But for the end-to-end framework it is beneficial to use a smaller number of iterations $T_1$ since this is computationally more efficient, and keeps the gradient $\nabla \hat{w}(u)$ from approaching zero. In the regime of smaller $T_1$, the ProjectNet method also has an advantage in terms of objective function value, with up to $12.5\%$ improvement. See Figure~\ref{fig:improvement-app}. 



\begin{figure}[htb]
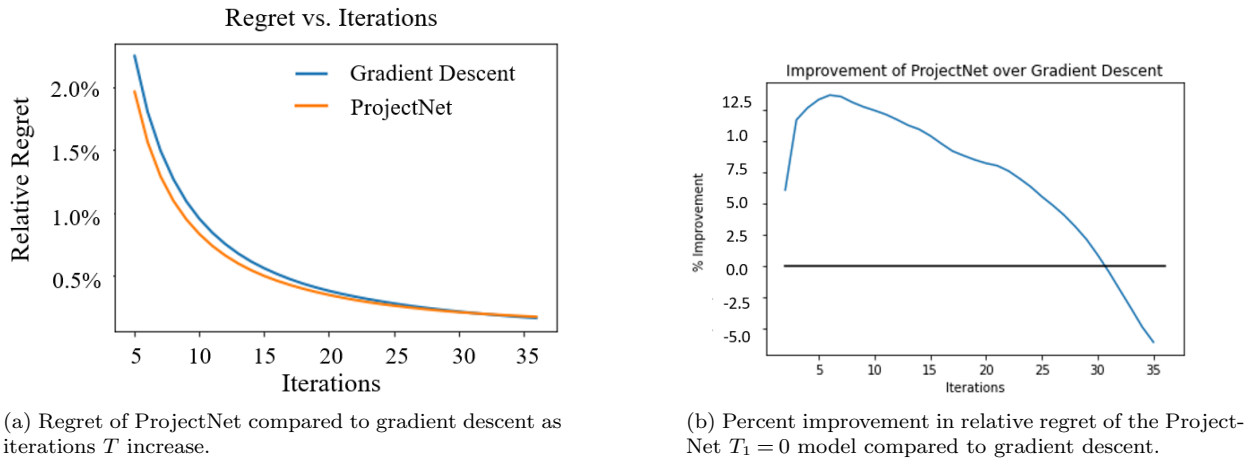

    \centering
\begin{subfigure}{.45\textwidth}
  \centering
    \includegraphics[scale=0.45]{pics/projectnet_vs_gd1.png}
    \caption[size=11pt]{Regret of ProjectNet compared to gradient descent as iterations $T$ increase. 
    }
    \label{fig:projectnet vs gd app}
\end{subfigure}\hfill
\begin{subfigure}{.45\textwidth}
    \includegraphics[scale=0.4]{pics/projectnet_improvement.png}
    \caption{Percent improvement in relative regret of the ProjectNet $T_1 = 0$ model compared to gradient descent. 
    }
    \label{fig:improvement-app}
\end{subfigure}
\caption{
Comparison of ProjectNet to Gradient Descent }
\label{fig:matching-app}
\end{figure}

\label{appendix:newsvendor}
\subsection{Multi-product newsvendor}
\label{sec:multi-product}

As an example, let us consider the multi-product newsvendor problem. 
We are given a total of $K$ products to allocate inventory for. Each one has a local demand realization that is random. There is a holding cost $h_j$ for each product $j=1,\dots,K$ (the cost paid for each unit of stock that remains unsold) and a lost sale cost $b_j$ (the cost for each unit of unmet demand) specific to each product. The objective of this problem is to decide how much inventory of product to allocate. The cost of decision $w$ and realization $u$ is given by 
\begin{equation}
    g_u(w) = \sum_{j=1}^K h_j (w_j - u_j)^+ + b_j(u_j - w_j)^+.
\end{equation}
We additionally impose a constraint on the total amount of stock $C$ that can be stored across all products. Given a known demand $u$, the nominal optimization problem is then given by the following:
\begin{equation}
\begin{array}{lll}
     w^*(u) =& \arg \min_{w \geq 0} &\ g_u(w) \\ \\
& \text{subject to}
        &\ \sum_{j=1}^k w_j \leq C.
\end{array}
\end{equation}
We assume we are given $N = 500$ datapoints $(x^n, u^n), n = 1, \dots, N$ of feature observations $x^n$ and corresponding demand observations $u^n$. Each $x^n$ is generated from a Gaussian distribution with zero covariance and random mean between [-1,1]. Then, we construct a random 2-layer neural network with ReLU activation. Passing in $x^n$ to this network generates the data for $u^n$.
The goal is to learn some forecasting function $f_\theta(x)$ which minimizes the in-sample cost exactly as in equation~(\ref{eq:end-to-end}).
To reiterate,  the key difficulty lies in taking the gradient of $w^*(u)$ with respect to $u$. Hence, we approximate this with our ProjectNet approach.

 We compare against four other methods. (1) A traditional predict-then-optimize approach which only predicts the uncertain parameters, independent of the optimization problem. (2) The OptNet framework of \citet{optnet} also used for end-to-end learning. This approach requires quadratic objectives, hence we add quadratic regularization terms to the objective as described in \citet{Wilder2019}. (3) The traditional sample average approximation (SAA) method which does not incorporate features. And (4) an extension of SAA to use feature information as proposed in \citet{Kallus2020}. In particular, we use a K-nearnest neighbor (KNN) method to determine the weights. Hyperparameters (such as $K$) are chosen by hyperaparameter tuning. {Next we describe the specific formulations that we use for t   hese other methods.}
 

In particular, we reformulate the optimization problem to use with OptNet as follows:
\begin{equation}
\begin{array}{lll}
     w_{\text{optnet}}^*(u) =& \arg \min_{w,p,q \geq 0} &\ h_j \cdot p_j + b_j \cdot q_j + \alpha (\norm{w}^2 + \norm{p}^2 + \norm{q}^2) \\ \\ 
& \text{subject to}
        &\ p_j \geq w_j - u_j , \quad q_j \geq u_j - w_j, \quad \sum_{j=1}^k w_j \leq C 
    \end{array}
\end{equation}
where we introduce variables $p_j,q_j$ to describe the amount of overstocked or understocked units in order to linearize the objective. We also introduce the regularization terms $\norm{w}^2 + \norm{p}^2 + \norm{q}^2$ to problem has nonzero gradient with respect to $u$. We chose a small value of $\alpha = 0.01$ so that it approximates the original problem well (indeed, at $\alpha = 0.01$, $w^*_{\text{optnet}}(u) = w^*(u)$). The training task to learn the forecasting function is now
\begin{equation}
    \min_\theta \sum_{n=1}^N g_{u^n}(w^*_{\text{optnet}}(f_\theta(x^n)))
\end{equation}
\noindent The SAA formulation is given by
\begin{equation}
\begin{array}{lll}
    &\min_{w \geq 0} &\ \sum_{n=1}^N g_{u^n}(w) \\ \\
& \text{subject to}
        &\ \sum_{j=1}^k w_j \leq C.
\end{array}
\end{equation}
Note that this problem does not depend on features $x$. The approach in \citep{Kallus2020} extends this by making use of features. In particular, instead of minimizing over all data $u^n$, we only minimize over the $k$-nearest neighbors to the out-of-sample features $x$. That is, given an out-of-sample point $x$, we calculate the decision
\begin{equation}
\begin{array}{lll}
   & \min_{w \geq 0} &\ z_n g_{u^n}(w) \\ \\
& \text{subject to}
        &\ \sum_{j=1}^k w_j \leq C,
\end{array}
\end{equation}
where $z_n = 1$ if $x^n$ is one of the $k$-nearest neighbors of $x$ and $z_n = 0$ otherwise.

The results of the experiment can be found in Table~\ref{table:runtime}. We observe that the end-to-end methods based on ProjectNet and OptNet takes better advantage of the problem structure to provide lower-cost decisions. Crucially, the end-to-end method based on ProjectNet is computationally more efficient, 10 times faster to train than the OptNet framework which needs to solve the original optimization problem at each iteration. There is a slight increase of at most 5\% in cost due to the nature of approximation of ProjectNet. This gap may potentially be further reduced by more parameter tuning.


\begin{table}[t]
\centering
\begin{tabular}{c c c c c c}
    \toprule 
    \# Products & SAA & Predict-then-optimize & SAA (KNN) & OptNet & ProjectNet  \\ \midrule 
    50 & $<$1s / 20.1 & 1.1s / 19.5 & $<$1s / 19.6 & 73s / 18.7  & 16s / 18.5 \\
    100 & $<$1s / 6.6 & 1.3s / 6.3 & $<$1s / 6.2 &  504s / 5.9 & 52s / 5.6   \\
    \bottomrule 
\end{tabular}
\caption{Running Time and Task-Based Cost Comparison. Left entry of each cell is the running time per epoch (models trained for the same number of epochs until convergence). The right entry is the average decision cost. 
}
\label{table:runtime}
\end{table}



\subsection{Optimality in the No-feature Case}
\label{subsection:newsvendor optimality}
In this example, we consider the case with no feature information in which we make the same single decision for any datapoint. In this particular case, we can find the exact optimal solution and compare against our proposed method using approximate projections. 
The experiment is as follows. Suppose we are given historical data $u^1, \dots, u^N$ of observed demand. Then, we wish to find the single optimal decision
\begin{equation}
\begin{array}{lll}
    & \min_{w \geq 0} &\ \sum_{n=1}^N g_{u^n}(w) \\ \\ 
& \text{subject to}
        &\ \sum_{j=1}^k w_j \leq C,
\end{array}
\end{equation}
through sample average approximation (SAA). Furthermore, SAA is guaranteed to converge to an optimal solution given enough samples from the underlying distribution \citep{shapiro2003monte}.
This problem is generally solved by traditional optimization methods. In this newsvendor case, this can be rewritten as a linear program. However, we may also solve this by gradient descent, ensuring feasibility by approximate projection on the constraint. Our problem becomes
\begin{equation}
    \min_w \sum_{n=1}^N g_{u^n}(\tilde{\pi}(w)),
\end{equation}
where $\tilde{\pi}$ is the approximate projection operator onto the constraints $\{ w \geq 0, \sum_{j=1}^K w_j \leq C \}$.  Experimentally, we find that there is an optimality gap of at most $0.1\%$ of the proposed approach over SAA showing that using approximate projections comes at minimal cost and gives rise to near optimal solutions. See Table~\ref{tab:no-feature} for more details. 

\begin{table}[t]
    \centering
\begin{tabular}{c c c}
    \toprule 
    Capacity & SAA  &  \multicolumn{1}{p{5cm}}{\centering Gradient Descent with \\ Approximate Projection} \\ \midrule 
    10 & 32.528 & 32.53 \\ 
    20 & 9.661 & 9.669 \\ 
    30 & 4.173 & 4.181 \\ 
    \bottomrule 
\end{tabular}
\caption{Comparison of SAA and gradient descent with approximate projections.}
\label{tab:no-feature}
\end{table}

\subsection{The Newsvendor Problem when Costs are Quadratic}

Finally, we consider the newsvendor problem with quadratic costs. In particular, the penalty of over-allocating or under-allocating product scales quadratically. We use an identical setting from \citet{donti2017task} which also considers this problem from an end-to-end perspective. In this setting, similarly to \citet{donti2017task}, we assume that the demand takes values over a discrete set of possible values $d_1, \dots, d_K$. Moreover, we make a distributional forecast $f_\theta(x)_k$ to determine the probability that given observed features $x$, that the demand is $d_k$. Given a single product with a realized demand of $d$ and a supply allocation of $w$, the objective value is given as follows,
\begin{equation}
g_d(w) = c_0 w + \frac{1}{2} q_0 w^2 + c_b (d - w)^+ + q_b ((d - w)^+)^2 + c_h (w - d)^+ + q_h ((w - d)^+)^2
\end{equation}
with different known holding and backorder costs $c_0, q_0, c_h, c_b, q_h, q_b$. Given a distributional forecast $ f_\theta(x) $, the optimal stocking quantity $w$ is given by
\begin{align}
    w^*(f_\theta(x)) = \arg \min_{w \geq 0} c_0 w + \frac{1}{2} q_0 w^2 + \sum_{k=1}^K f_\theta(x)_k  \bigl( &\ c_b (d - w)^+ + q_b ((d - w)^+)^2 + \\ &\ c_h (w - d)^+ + q_h ((w - d)^+)^2 \bigr). \nonumber
\end{align}
Note that this formulation is a piece-wise quadratic optimization problem. 
Finally, given data $(x^1, d^1), \dots, (x^N, d^N)$, one learns $f_\theta(x)$ by the following risk-minimization problem:
$
    \min_{\theta} \sum_{n=1}^N g_{d^n}(w^*(f_\theta(x^n)))$.
As before, we approximate solutions  $w^*(p)$ for a distribution $p$, by using the ProjectNet architecture. We compare our method against the one of \citet{donti2017task} where this same optimization problem was posed. See Table~\ref{tab:quad newsvendor} for a comparison of training times and average decision cost of each method. First, note that as this problem is quadratic, the OptNet framework used in \citet{donti2017task} can exactly calculate the (non-zero) gradients of the problem. But again, we see the ProjectNet method is faster by a factor of 2 and scales better with increasing problem size while achieving  slightly lower cost. 

\begin{table}[H]
    \centering
    \begin{tabular}{c c c c} \toprule
        $K$ (possible demands) & Predict-then-optimize & \citet{donti2017task} & ProjectNet \\ \midrule
         5 & 15.98 / $<$ 0.1s & 13.30 / 4.2s & 13.07 / 2.4s \\ 
         10 &  28.40 / $<$ 0.1s & 26.0 / 9.9s & 25.65 / 4.2s  \\ \bottomrule
    \end{tabular}
    \caption{Performance comparison on quadratic newsvendor. Left side of each column represents average cost on test set, and right side represents average training time over 100 epochs.}
    \label{tab:quad newsvendor}
\end{table}

\section{Visualizing Toy Examples}

In what follows, we will gain some intuition behind the output structure of the ProjectNet approach by considering some low-dimensional toy examples that are as a result easy to visualize. Subsequently in the next section, we will generalize these observations and also provide analytical results. 

First, we investigate the output of the model given various different input cost vectors. Consider the following optimization problem we wish to learn via the ProjectNet:
\begin{equation}
\label{basic}
\begin{array}{lll}
w^*(u) =& \arg \min_{w} & u_1 w_1 + u_2 w_2 \\ \\
&\text{subject to} & w_1 + 2 w_2 \geq 1 , \quad 2w_1 + w_2 \geq 1, \quad w_1,w_2 \geq 0.
\end{array}
\end{equation}
As data, we generate random data $u^1, \dots, u^N$ with each $u^n = [ u^n_1 \ u^n_2]$ being drawn uniformly at random from the unit square. We use this to train a ProjectNet model $\hat{w}$. We then examine the output on ``test'' vectors that are chosen uniformly spaced along a circle. In particular, $M$ cost vectors $u^m$ defined as 
\begin{equation}
    u^m = \paren{\cos \paren{m\cdot \frac{\pi}{2M} }, \sin \paren{m\cdot \frac{\pi}{2M} } }.
\end{equation}
\noindent In Figure~\ref{fig:progression}, we plot the values of $\hat{w}(u^m)$.  
We see that the output of a trained ProjectNet is concentrated around vertices and continuously transitions from one vertex to an adjacent vertex as the cost vector changes. We want solutions to be concentrated around vertices as those are the optimal solutions, but in order to ensure the gradient is nonzero, the property that the ProjectNet's output transitions continuously between vertices is crucial.

Next, we also present the sequence of steps taken at each iteration of the ProjectNet model. In particular, we compare against a traditional projected gradient descent approach. We consider a different optimization problem in three variables:
\begin{equation}
\label{basic2}
\begin{array}{lll}
w^*(u) =& \arg \min_{w} & u_1 w_1 + u_2w_2 + u_3w_3 \\ \\ 
&\text{subject to} & w_1 + w_2 + w_3 = 1, \quad w_1, w_2, w_3 \geq 0.
\end{array}
\end{equation}
Again, we generate random cost vectors and train a ProjectNet model $\hat{w}$ to approximate $w^*(u)$. In Figure~\ref{fig:projectnet vs gd path} we notice that the ProjectNet method is able to converge faster by making use of the learned component of the update rule $L(w_t)$. 

\begin{figure}[t]
  \centering
\includegraphics[scale=0.5]{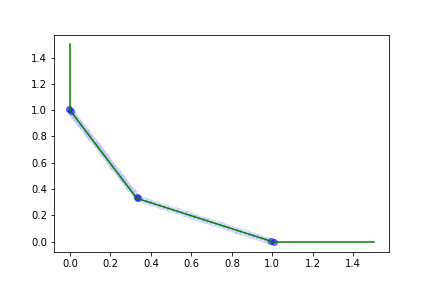} 
\caption{ProjectNet output varies continuously between vertices.}
\label{fig:progression}
\end{figure}

\begin{figure}[tb]
    \centering
\includegraphics[scale=0.55]{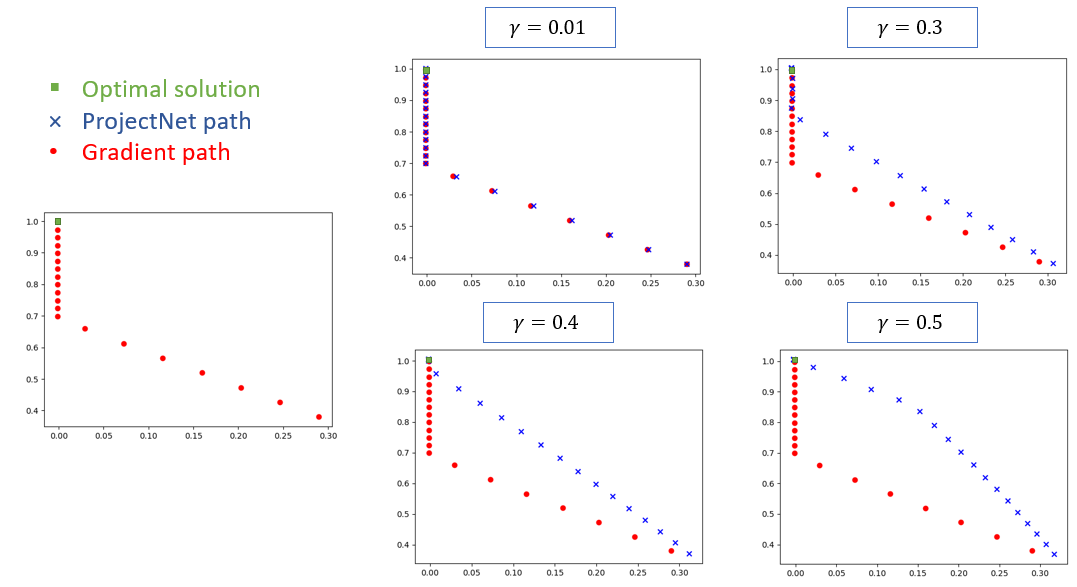} 
\caption{Comparison of the path taken in gradient descent versus the ProjectNet model. 
}
\label{fig:projectnet vs gd path}
\end{figure}

\end{APPENDICES}

\end{document}